\documentclass{article}
\usepackage{ijcai25}

\usepackage{times}
\usepackage{soul}
\usepackage{url}
\usepackage[hidelinks]{hyperref}
\usepackage[utf8]{inputenc}
\usepackage[small]{caption}
\usepackage{graphicx}
\usepackage{amsmath}
\usepackage{amsthm}
\usepackage{booktabs}
\usepackage{algorithm}
\usepackage{algorithmic}
\usepackage[switch]{lineno}

\usepackage{amsfonts,bm}
\DeclareMathOperator*{\argmax}{arg\,max}

\usepackage{bbm}
\usepackage{subcaption}
\usepackage{array,multirow}
\usepackage{adjustbox}
\usepackage{extarrows}

\usepackage{makecell}
\usepackage[english]{babel}
\usepackage{amsthm}
\usepackage[dvipsnames]{xcolor}

\theoremstyle{plain}
\newtheorem{theorem}{Theorem}[section]

\newtheorem{lemma}[theorem]{Lemma}

\theoremstyle{definition}
\newtheorem{definition}[theorem]{Definition}

\theoremstyle{remark}

\usepackage{hyperref}
\title{Explainable Graph Representation Learning via Graph Pattern Analysis}

\author{
Xudong Wang$^{1}$
\and
Ziheng Sun$^{1,2}$ \and
Chris Ding$^{1}$\And
Jicong Fan$^{1,2,}$\footnote{Corresponding author.\\
This is the full version of the paper published in the 
\emph{Proceedings of the Thirty-Fourth International Joint Conference on Artificial Intelligence (IJCAI-25)}, 
Main Track, pages~3426--3434.\\
\href{https://doi.org/10.24963/ijcai.2025/381}{https://doi.org/10.24963/ijcai.2025/381}}\\
\affiliations
$^1$School of Data Science, The Chinese University of Hong Kong, Shenzhen (CUHK-Shenzhen), China\\
$^2$Shenzhen Research Institute of Big Data, Shenzhen, China\\
\emails
\{xudongwang, zihengsun\}@link.cuhk.edu.cn,
\{chrisding, fanjicong\}@cuhk.edu.cn
}
\begin{document}

\maketitle

\begin{abstract}
Explainable artificial intelligence (XAI) is an important area in the AI community, and interpretability is crucial for building robust and trustworthy AI models. While previous work has explored model-level and instance-level explainable graph learning, there has been limited investigation into explainable graph representation learning. In this paper, we focus on representation-level explainable graph learning and ask a fundamental question: What specific information about a graph is captured in graph representations? Our approach is inspired by graph kernels, which evaluate graph similarities by counting substructures within specific graph patterns. Although the pattern counting vector can serve as an explainable representation, it has limitations such as ignoring node features and being high-dimensional. To address these limitations, we introduce a framework (PXGL-GNN) for learning and explaining graph representations through graph pattern analysis. We start by sampling graph substructures of various patterns. Then, we learn the representations of these patterns and combine them using a weighted sum, where the weights indicate the importance of each graph pattern's contribution. We also provide theoretical analyses of our methods, including robustness and generalization. In our experiments, we show how to learn and explain graph representations for real-world data using pattern analysis. Additionally, we compare our method against multiple baselines in both supervised and unsupervised learning tasks to demonstrate its effectiveness.
\end{abstract}

{\setlength\abovedisplayskip{4pt}
\setlength\belowdisplayskip{0pt}
\section{Introduction}
The research field of explainable artificial intelligence (XAI) \cite{adadi2018peeking,angelov2021explainable,hassija2024interpreting} is gaining significant attention in both AI and science communities. Interpretability is crucial for creating robust and trustworthy AI models, especially in critical domains like transportation, healthcare, law, and finance. Graph learning \cite{sun2024lovasz,sun2024mmd,NEURIPS2024_f900437f} is an important area of AI that particularly focuses on graph-structured data that widely exists in social science, biology, chemistry, etc. Explainable graph learning (XGL) \cite{kosan2023gnnx} can be generally classified into two categories: model-level methods and instance-level methods.

Model-level methods of XGL provide transparency by analyzing the model behavior. Examples include XGNN \cite{yuan2020xgnn}, GLG-Explainer \cite{azzolin2022global}, and GCFExplainer \cite{huang2023global}. Instance-level methods of XGL offer explanations tailored to specific predictions, focusing on why particular instances are classified in a certain manner. For instance, GNNExplainer \cite{ying2019gnnexplainer} identifies a compact subgraph structure crucial for a GNN's prediction. PGExplainer \cite{luo2020parameterized} trains a graph generator to incorporate global information and parameterize the explanation generation process. AutoGR \cite{wang2021explainable} introduces an explainable AutoML approach for graph representation learning. 

However, these works mainly focus on enhancing the transparency of GNN models or identifying the most important substructures that contribute to predictions. The exploration of representation-level explainable graph learning (XGL) is limited. We propose explainable graph representation learning and ask a fundamental question: \textbf{What specific information about a graph is captured in graph representations?}
Formally, if we represent a graph $G$ as a $d$-dimensional vector $\bm{g}$, our goal is to understand what specific information about the graph $G$ is embedded in the representation $\bm{g}$. This problem is important and has practical applications. Some graph patterns are highly practical and crucial in various real-world tasks, and we want this information to be captured in representations. For instance, in molecular chemistry, bonds between atoms or functional groups often form cycles (rings), which indicate a molecule’s properties and can be used to generate molecular fingerprints \cite{morgan1965generation,alon2008biomolecular,rahman2009small,o2016comparing}. Similarly, cliques characterize protein complexes in Protein-Protein Interaction networks and help identify community structures in social networks \cite{girvan2002community,jiang2010finding,fox2020finding}.

Although some previous works such as \cite{kosan2023gnnx} aimed to find the most critical subgraph $S$ by solving optimization problems based on perturbation-based reasoning, either factual or counterfactual, this kind of approach assumes that the most important subgraph $S$ mainly contributes to the representation $\bm{g}$, neglecting other aspects of the graph, which doesn't align well with our goal of thoroughly understanding graph representations. Analyzing all subgraphs of a graph $G$ is impractical due to their vast number. To address the challenge, we propose to group the subgraphs into different graph patterns, like paths, trees, cycles, cliques, etc, and then analyze the contribution of each graph pattern to the graph representation $\bm{g}$.

Our idea of pattern analysis is inspired by graph kernels, which compare substructures of specific graph patterns to evaluate the similarity between two graphs \cite{kriege2020survey}. For example, random walk kernels \cite{borgwardt2005protein,gartner2003graph} use path patterns, sub-tree kernels \cite{da2012tree,smola2002fast} examine tree patterns, and graphlet kernels \cite{prvzulj2007biological} focus on graphlet patterns. The graph kernel involves learning a pattern counting representation vector $\bm{h}$, which counts the occurrences of substructures of a specific pattern within the graph $G$. While the pattern counting vector $\bm{h}$ is an explainable representation, it has some limitations, such as the high dimensionality and ignorance of node features. 

There also exist some representation methods based on subgraphs and substructures, such as Subgraph Neural Networks (SubGNN) \cite{kriege2012subgraph}, Substructure Assembling Network (SAN) \cite{zhao2018substructure}, Substructure Aware Graph Neural Networks (SAGNN) \cite{zeng2023substructure}, and Mutual Information (MI) Induced Substructure-aware GRL \cite{wang2020exploiting}. However, these methods mainly focus on increasing expressiveness and do not provide explainability for representation learning. 

In this work, we propose a novel framework to learn and explain graph representations via graph pattern analysis. We start by sampling graph substructures of various patterns. Then, we learn the representations of these patterns and combine them adaptively, where the weights indicate the importance of each graph pattern's contribution. We also provide theoretical analyses of our methods, including robustness and generalization. Additionally, we compare our method against multiple baselines in both supervised and unsupervised learning tasks to demonstrate its effectiveness and superiority.
Our contributions are summarized as follows:
\begin{itemize}
    \item Unlike previous model-level and instance-level XGL, we introduce a new issue --- representation-level explainable graph learning. This issue focuses on understanding what specific information about a graph is embedded within its representations.
    \item We propose two strategies to learn and explain graph representations, including a graph ensemble kernel method~(\textbf{PXGL-EGK}) and a pattern analysis GNN method~(\textbf{PXGL-GNN}). The latter involves using GNNs to learn the representations of each pattern and evaluate its contribution to the ensemble graph representation.
    \item We provide theoretical analyses of our methods, including robustness and generalization.
\end{itemize}

\section{Notations}
In this work, we use $x$, $\bm{x}$, $\bm{X}$, and $\mathcal{X}$ (or $X$) to denote scalar, vector, matrix, and set, respectively. We denote $[n] = \{1, 2, ..., n\}$. Let $G = (V, E)$ be a graph with $n$ nodes and $d$-dimensional node features $\{\bm{x}_v \in \mathbb{R}^d \mid v \in V\}$. We denote $\bm{A} \in \{0, 1\}^{n \times n}$ the adjacency matrix and $\bm{X} = [\bm{x}_1, \ldots, \bm{x}_n]^\top \in \mathbb{R}^{n \times d}$ the node features matrix. Let $\mathcal{G} = \{G_1, \dots, G_N\}$ be a dataset of $N$ graphs belonging $C$ classes, where $G_i = (V_i, E_i)$. For $G_i$, we denote its number of nodes as $n_i$, the one-hot graph label as $\bm{y}_i \in \{0, 1\}^C$, the graph-level representation as a vector $\bm{g}_i \in \mathbb{R}^d$, the adjacency matrix as $\bm{A}_i$, and the node feature matrix as $\bm{X}_i$.
Let $S = (V_{S}, E_{S})$ be a subgraph of graph $G = (V, E)$ such that $V_{S} \subseteq V$ and $E_{S} \subseteq E$. The the adjacency matrix of $S$ is denoted as $\bm{A}_S \in \{0, 1\}^{|V_{S}|\times |V_{S}|}$ and the node feature matrix of $S$ is sampled from the rows of $\bm{X}$,  denoted as  $\bm{X}_S \in \mathbb{R}^{|V_{S}|\times d}$.

The graph pattern is defined as a set of all graphs that share certain properties, denoted as $\mathcal{P} = \{P_1, P_2, \ldots, P_i, \ldots\}$, where $P_i$ is the $i$-th example of this pattern. In this work, the graph patterns are basic graph families such as paths, trees, cycles, cliques, etc.  For example:
\begin{itemize}
    \item $\mathcal{P}_{\text{path}} = \{\text{ph}_1, \text{ph}_2, \ldots, \text{ph}_i, \ldots\}$ is a path pattern with $\text{ph}_i$ as a path of length $i$.
    \item $\mathcal{P}_{T} = \{T_1, T_2, \ldots, T_i, \ldots\}$ is a tree pattern where $T_i$ is the $i$-th tree.
    \item $\mathcal{P}_{\text{gl}} = \{\text{gl}_1, \text{gl}_2, \ldots, \text{gl}_i, \ldots\}$ is a graphlet pattern where $\text{gl}_i$ is the $i$-th graphlet.
\end{itemize}
\begin{figure*}[htb]
    \centering
    \includegraphics[width = \linewidth,trim={0 15 0 0},clip]{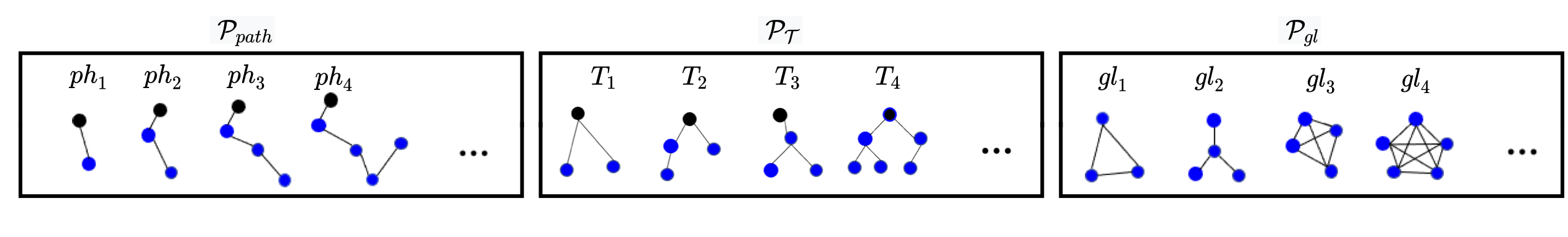}
    \caption{Examples of graph patterns:~$\mathcal{P}_{\text{path}}$, $\mathcal{P}_{\text{T}}$ and $\mathcal{P}_{\text{gl}}$}
    \label{fig:graphpatternexamples}
\end{figure*}
Figure~\ref{fig:graphpatternexamples} illustrates some intuitive examples of graph patterns. Notably, there are overlaps among different patterns; for instance, the graph $T_3 \in \mathcal{P}_{T}$ and $\text{gl}_2 \in \mathcal{P}_{\text{gl}}$ are identical, being both a tree and a graphlet. Overlaps are inevitable due to the predefined nature of these basic graph families in graph theory. 
We denote a set of $M$ different patterns as $\{ \mathcal{P}_1, \mathcal{P}_2, \ldots, \mathcal{P}_m, \ldots, \mathcal{P}_M\}$. Given the pattern $\mathcal{P}_m$ and the graph $G_i$, 
the pattern sampling set is denoted as $\mathcal{S}_i^{(m)}$ and the pattern representation is denoted as $\bm{z}^{(m)}_i \in \mathbb{R}^d$.

\section{XGL via Ensemble Graph Kernel}
In this section, we learn and explain the pattern counting graph representation via graph kernels. 

\subsection{Pattern Counting Kernel}
A graph kernel $K: \mathbb{G} \times \mathbb{G} \rightarrow \mathbb{R}$ aims to evaluate the similarity between two graphs. Let $G_i$ and $G_j$ be two graphs in the graph dataset $\mathcal{G}$ and let $\mathcal{H}$ be a high-dimensional vector space. The key to a graph kernel is defining a mapping from the graph space to the high-dimensional vector space as $\phi: \mathbb{G} \rightarrow \mathcal{H}$, where $\bm{h}_i = \phi(G_i)$ and $\bm{h}_j = \phi(G_j)$. Then, the graph kernel can be defined as the inner product of $\bm{h}_i$ and $\bm{h}_j$, i.e., $K(G_i, G_j):= \bm{h}_i^\top \bm{h}_j$. The most widely used mapping $\phi$ is the one counting the occurrences of each example in the pattern $\mathcal{P}$ within graph $G$. The corresponding pattern counting vector is defined as follows:
\begin{definition}[Pattern Counting Vector]\label{def:pc-vec}
Given a graph $G$ and a pattern $\mathcal{P} = \{P_1, P_2, \ldots, P_i, \ldots\}$, a pattern counting mapping $\phi: \mathbb{G} \rightarrow \mathcal{H}$ is defined as 
\begin{equation}\label{eqn:pc-vec}
    \bm{h} = \phi(G; \mathcal{P}), ~~\text{with}~~ \bm{h} = [h^{(1)}, h^{(2)}, \ldots, h^{(i)}, \ldots],
\end{equation}
where $h^{(i)}$ is the number of occurrences of pattern example $P_i$ as a substructure within graph $G$. We call $\bm{h}$ a pattern counting vector of $G$ related to pattern $\mathcal{P}$. 
\end{definition}
Then the pattern counting kernel $K_{\mathcal{P}}: \mathbb{G} \times \mathbb{G} \rightarrow \mathbb{R}$ based on pattern $\mathcal{P}$ is defined as:
\begin{definition}[Pattern Counting Kernel]\label{def:pc-ker}
Given the a pattern counting mapping $\phi(G; \mathcal{P})$, a pattern counting kernel is defined as 
\begin{equation}\label{eqn:pc-ker}
    K_{\mathcal{P}}(G_i, G_j) := \langle \phi(G_i; \mathcal{P}), \phi(G_j; \mathcal{P}) \rangle 
        = \bm{h}_i^\top \bm{h}_j
\end{equation}
\end{definition}
The pattern counting kernel $K_{\mathcal{P}}$ is uniquely determined by the pattern $\mathcal{P}$. For example, if $\mathcal{P}$ is selected as the path pattern $\mathcal{P}_{\text{path}}$, we obtain a random walk kernel \cite{borgwardt2005protein,gartner2003graph}. If $\mathcal{P}$ is the tree pattern $\mathcal{P}_{T}$, we get a sub-tree kernel \cite{da2012tree,smola2002fast}. Similarly, if $\mathcal{P}$ is the graphlet pattern $\mathcal{P}_{\text{gl}}$, we derive a graphlet kernel \cite{prvzulj2007biological}. 

\subsection{Pattern Analysis Using Graph Kernels}
Let $\{ \mathcal{P}_1, \mathcal{P}_2, \ldots, \mathcal{P}_M\}$ be a set of $M$ different graph patterns. For instance, $\mathcal{P}_1$ represents the path pattern and $\mathcal{P}_2$ represents the tree pattern. Then, we can define a set of $M$ different graph kernels as $\{ K_{\mathcal{P}_1}, K_{\mathcal{P}_2}, \ldots, K_{\mathcal{P}_M} \}$. Since the pattern counting kernel $K_{\mathcal{P}_m}$ is uniquely determined by the pattern $\mathcal{P}_m$, we can analyse the importance of pattern $\mathcal{P}_m$ by evaluating the importance of its pattern counting kernel $K_{\mathcal{P}_m}$. To achieve this, we define a learnable ensemble kernel as follows:
\begin{definition}[Learnable Ensemble Kernel]\label{def:ensemble-ker}
Let $\bm{\lambda} = [\lambda_1, \lambda_2, ..., \lambda_m,..., \lambda_M]^\top$ be a positive weight parameter vector. The ensemble kernel matrix $\bm{K}(\bm{\lambda}) \in \mathbb{R}^{|\mathcal{G}| \times |\mathcal{G}|}$ is defined as the weighted sum of $M$ different kernels $\{ K_{\mathcal{P}_1}, K_{\mathcal{P}_2}, \ldots, K_{\mathcal{P}_M} \}$. Given two graphs $G_i$ and $G_j$ in $\mathcal{G}$, the element at the $i$-th row and $j$-th column of $\bm{K}(\bm{\lambda})$ is given by
\begin{equation}
    \label{eqn:ensemble-ker}
    \begin{aligned}
        & K_{ij}(\bm{\lambda}) := \sum_{m = 1}^M \lambda_m ~ K_{\mathcal{P}_m}(G_i, G_j), ~~ \text{s.t}~~ \sum_{m=1}^M \lambda_m = 1,\\~~&\text{and}~~ \lambda_m\ge 0, ~~\forall m \in [M]. 
    \end{aligned}
\end{equation}
\end{definition}
Here, the weight parameter $\lambda_m$ indicates the importance of the kernel $K_{\mathcal{P}_m}$ as well as the corresponding graph pattern $\mathcal{P}_m$ within the dataset $\mathcal{G}$. Instead of the constrained optimization \eqref{eqn:ensemble-ker}, we may consider replacing $\lambda_m$ with $\exp(w_m)/\sum_{m=1}^M\exp(w_m)$ such that the constraints are satisfied inherently, which leads to an unconstrained optimization in terms of $\bm{w} = [w_1,\ldots,w_M]^\top$. In the following context, for convenience, we focus on \eqref{eqn:ensemble-ker}, though all results are applicable to the unconstrained optimization. To obtain the weight parameter $\bm{\lambda}$, we provide the supervised and unsupervised loss functions as follows.

\paragraph{Supervised Contrastive Loss.} Following \cite{oord2018representation}, given a kernel matrix $\bm{K} \in \mathbb{R}^{N \times N}$, we define the supervised InfoNEC loss of $\bm{K}$ as follows 
\begin{equation}
    \label{eqn:loss-scl}
    \begin{aligned}
        &\mathcal{L}_{\text{SCL}}(\bm{K}(\bm{\lambda})) = - \sum_{i \neq j} \mathbb{I}_{[\bm{y}_i = \bm{y}_j]} 
        (\log K_{ij}(\bm{\lambda}) \\
        &- \log [\sum_{k} \mathbb{I}_{[\bm{y}_i = \bm{y}_k, i \neq k] }  
        K_{ik}(\bm{\lambda}) + \mu \sum_{k} \mathbb{I}_{[\bm{y}_i \neq \bm{y}_k]} K_{ik}(\bm{\lambda}) ] ),
    \end{aligned}
\end{equation}
where $\mathbb{I}_{[\cdot]}$ is an indicator function and $\mu>0$ is a hyperparameter. 

\paragraph{Unsupervised KL Divergence.} Inspired by \cite{xie2016unsupervised}, given a kernel matrix $\bm{K} \in \mathbb{R}^{N \times N}$, we define the unsupervised KL divergence loss as follows
\begin{equation}
    \label{eqn:loss-kl-ker}
    \begin{aligned}
        & \mathcal{L}_{\text{KL}}(\bm{K}(\bm{\lambda})) = \mathbb{K} \mathbb{L} (\bm{K}(\bm{\lambda}), \bm{K}'(\bm{\lambda})), \\
        &~~\text{with}~~ 
        \bm{K}_{ij}'(\bm{\lambda}) = \frac{K_{ij}^2(\bm{\lambda}) / r_j}{\sum_{j'} K_{ij'}^2(\bm{\lambda}) / r_{j'}} ~~\text{and}~~ r_j = \sum_{j} K_{ij}(\bm{\lambda}),
    \end{aligned}
\end{equation}
where $r_j$ are soft cluster frequencies. By minimizing the KL divergence, the model adjusts the parameters $\boldsymbol{\lambda}$ to more accurately represent the natural clustering property of the dataset.

We use the $\mathcal{L}_{\text{SCL}}$ or $\mathcal{L}_{\text{KL}}$ as our loss function, i.e., 
$\mathcal{L}_{\text{ker}}(\bm{\lambda}) = \mathcal{L}_{\text{SCL}}(\bm{K}(\bm{\lambda})) ~~\text{or}~~ \mathcal{L}_{\text{KL}}(\bm{K}(\bm{\lambda}))$, when the graphs are labeled or unlabeled.
Then the weight parameter $\bm{\lambda}$ can be obtain by solving
\begin{equation}
    \label{eqn:solve-ker-lambda}
    \begin{aligned}
    \bm{\lambda}^* = \mathop{\textup{argmin}}{_{\bm{1}_M^\top \bm{\lambda} = 1,~ \bm{\lambda} \ge 0}}
                     ~~~\mathcal{L}_{\text{ker}}(\bm{\lambda}),
    \end{aligned}
\end{equation}
where $\bm{\lambda}^* = [\lambda_1^*, ..., \lambda_m^*, ... \lambda_M^*]^\top$ and $\lambda_m^*$ indicates the importance of kernel $K_{\mathcal{P}_m}$ as well as pattern $\mathcal{P}_m$. In Figure \ref{fig:ker-pattern-protein}, we can see that the ensemble Kernel performs better than each single kernel and the pattern analysis identifies the importance of each kernel as well as the related graph pattern. We call this method pattern-based XGL with ensemble graph kernel, abbreviated as \textbf{PXGL-EGK}. This method not only yields explainable similarity learning but also provides an approach to selecting graph kernels and their hyperparameters automatically if we consider different kernel types with different hyperparameters.

\begin{figure*}[ht]
    \centering
    \begin{subfigure}[b]{0.135\textwidth}
        \centering
        \caption{$\boldsymbol{\lambda}$}
        \includegraphics[width=\textwidth]{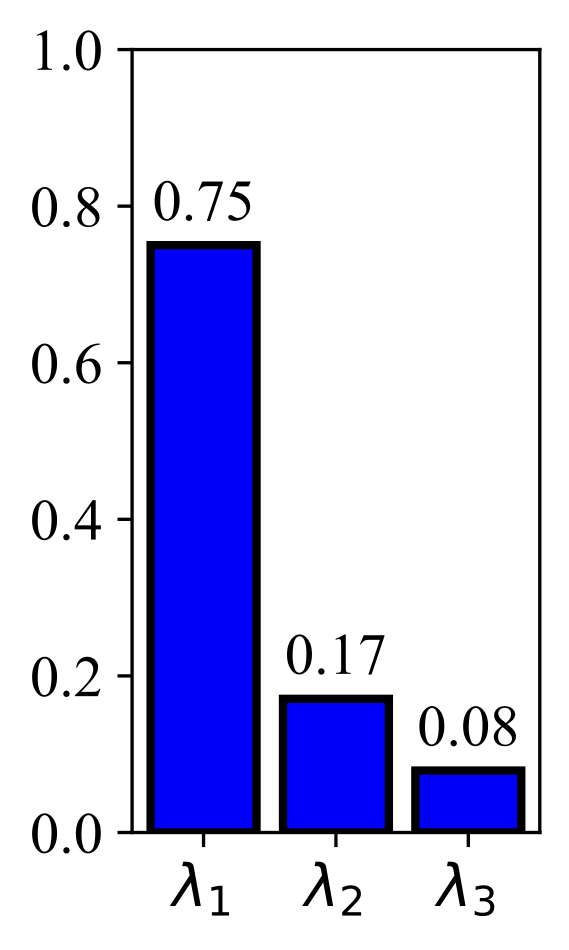}\\
        \label{fig:K_combined_lambda}
    \end{subfigure}
    \hfill
    \begin{subfigure}[b]{0.2\textwidth}
        \centering
        \caption{$K(\lambda)$: ensemble}
        \includegraphics[width=\textwidth]{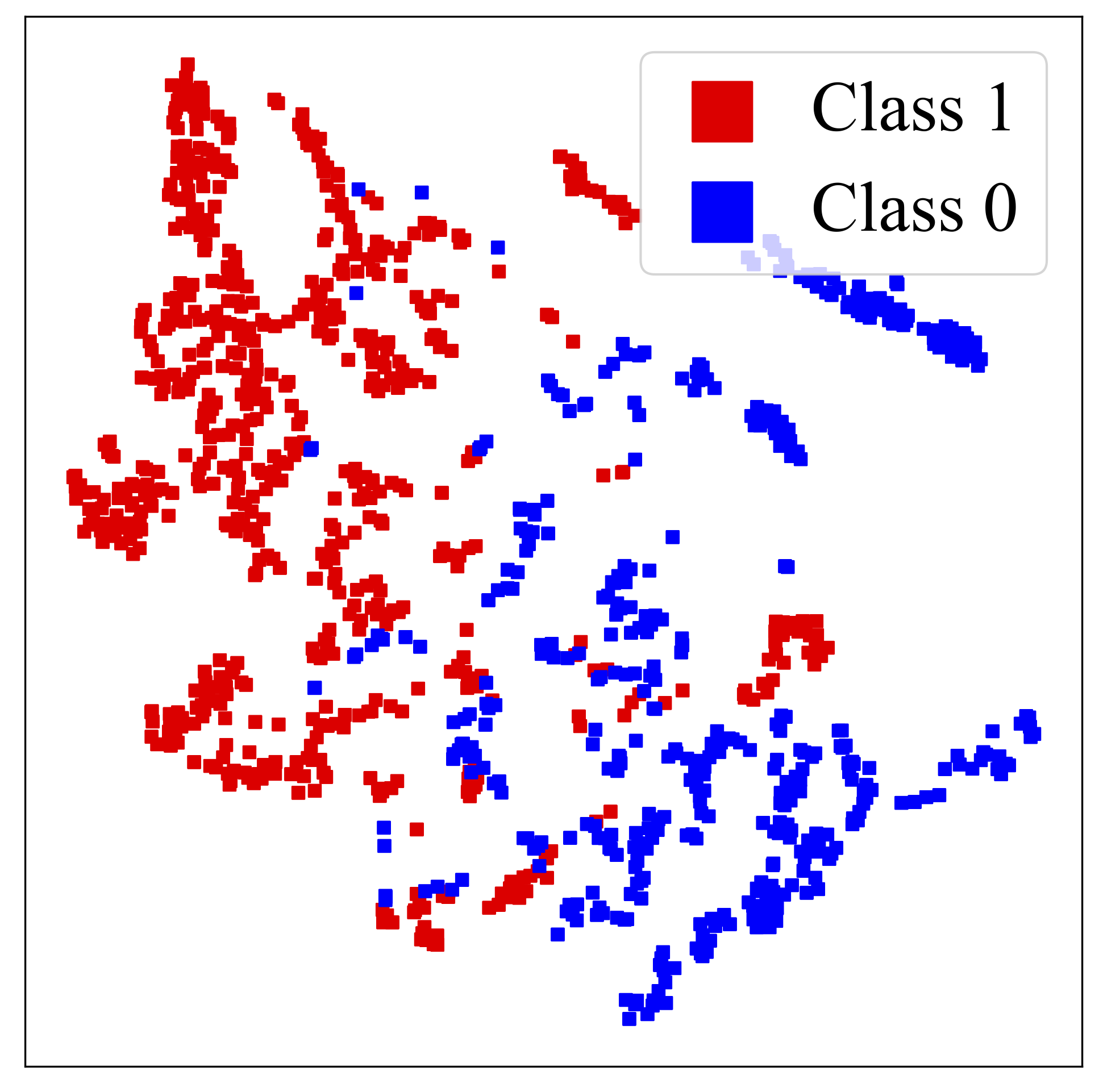}\\
        \footnotesize{$\sum_{m = 1}^M \lambda_m ~ K_{\mathcal{P}_m}$}
        \label{fig:K_combined}
    \end{subfigure}
    \hfill
    \begin{subfigure}[b]{0.2\textwidth}
        \centering
        \caption{$K_{\mathcal{P}_1}$: path}
        \includegraphics[width=\textwidth]{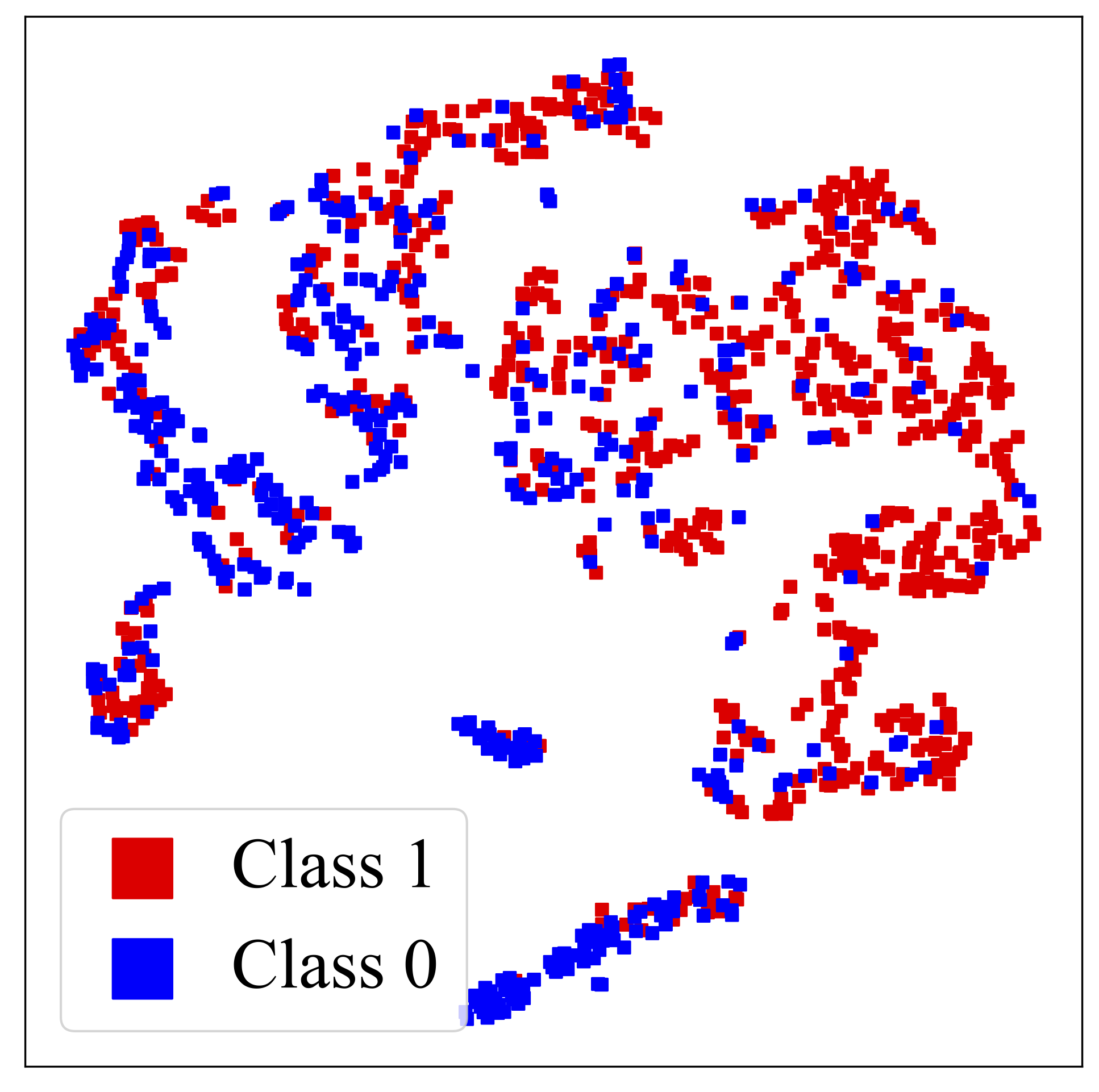}\\
        \footnotesize{$(\lambda_1 = 0.7502)$}
        \label{fig:K_rw}
    \end{subfigure}
    \hfill
    \begin{subfigure}[b]{0.2\textwidth}
        \centering
        \caption{$K_{\mathcal{P}_2}$: tree}
        \includegraphics[width=\textwidth]{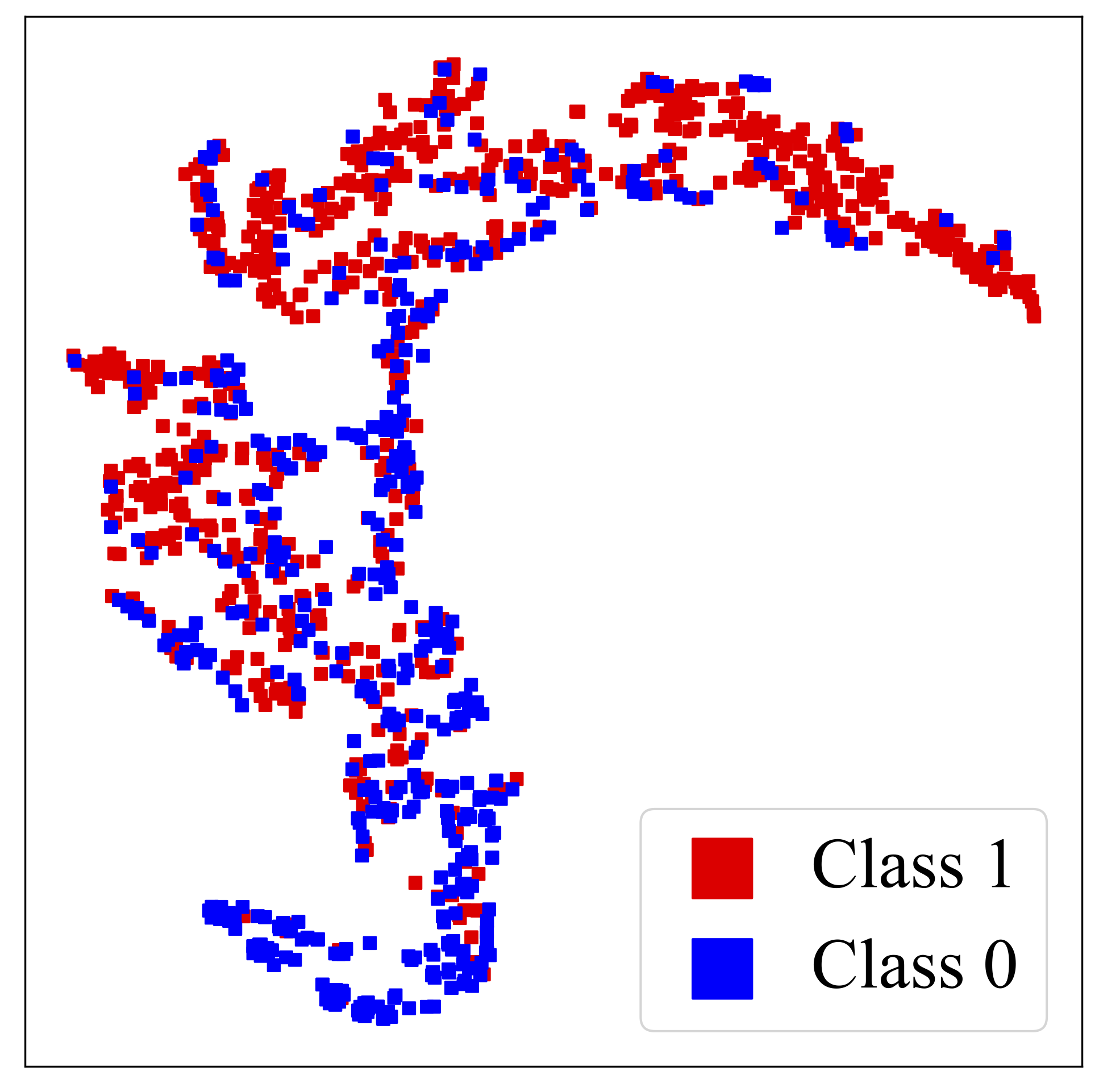}
        \footnotesize{$(\lambda_2 = 0.1707)$}
        \label{fig:K_subtree}
    \end{subfigure}
    \hfill
    \begin{subfigure}[b]{0.2\textwidth}
        \centering
        \caption{$K_{\mathcal{P}_3}$: graphlet}
        \includegraphics[width=\textwidth]{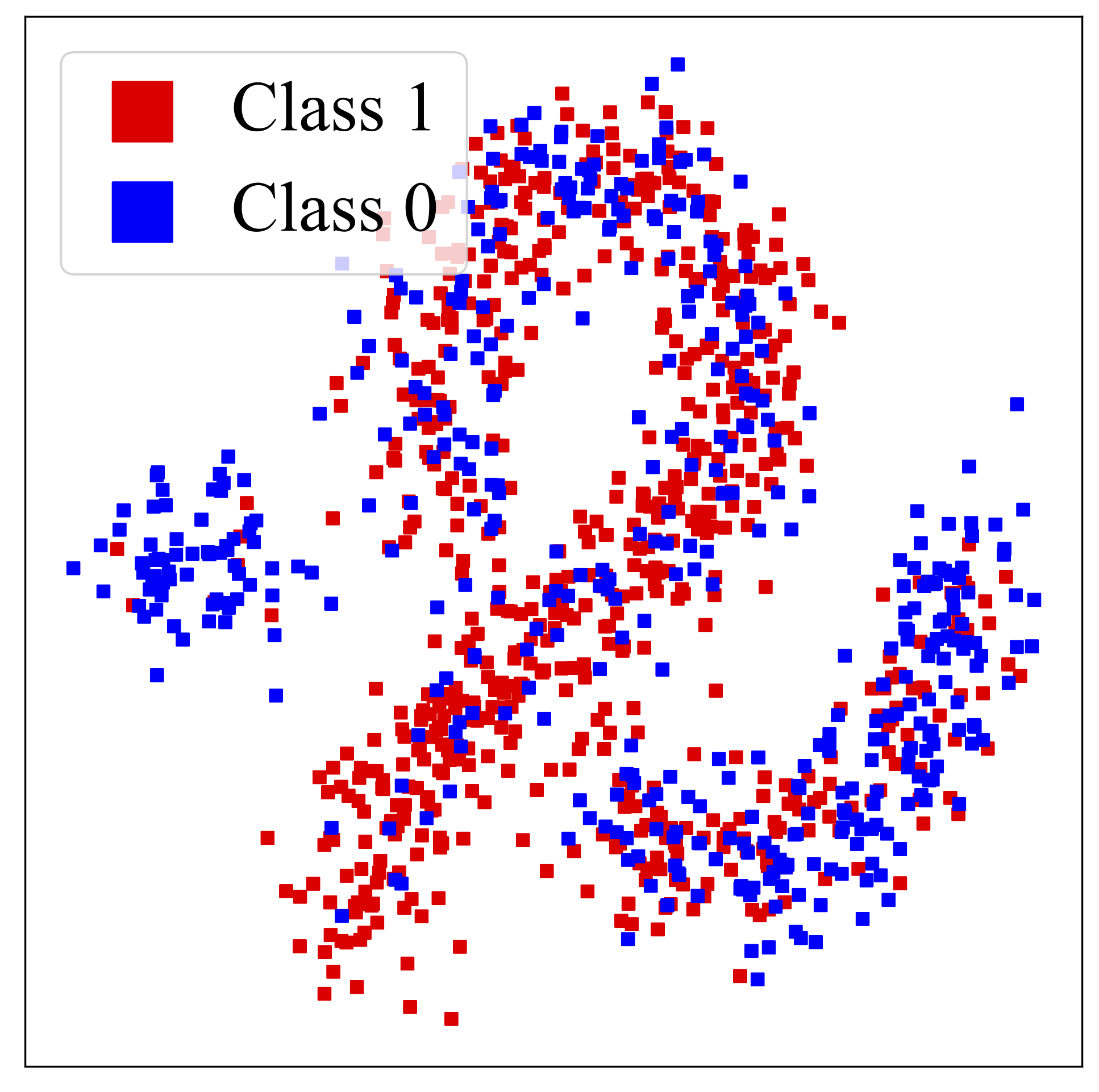}\\
        \footnotesize{$(\lambda_3 = 0.07912)$}
        \label{fig:K_graphlet}
    \end{subfigure}
    \caption{t-SNE visualizations of PXGL-EGK's different kernel embeddings for the dataset PROTEINS.}
    \label{fig:ker-pattern-protein}
\end{figure*}

\subsection{Limitations of Pattern Counting Vector}\label{sec_limit}
The pattern counting vector $\bm{h}$ from Definition \ref{def:pc-vec} is easy to understand and its importance can be evaluated using the weight parameter $\bm{\lambda}^*$ from \eqref{eqn:solve-ker-lambda}. However, it cannot directly explain the representation of graph $G$ due to the following limitations:
\begin{itemize}
    \item \textbf{Ignoring Node Features:} $\bm{h}$ captures the topology of $G$ but ignores node features $\bm{X}$. As shown by previous GNN works, node features are crucial for learning graph representations.
    \item \textbf{High Dimensionality:} The pattern set $\mathcal{P} = \{P_1, P_2, \ldots, P_i, \ldots\}$ can be vast, making $\bm{h}$ high-dimensional and impractical for many tasks. 
    \item \textbf{Time Complexity:} Counting patterns $P_i$ in $G$ is time-consuming due to the large number of patterns in $\mathcal{P}$. The function $\phi(G; \mathcal{P})$ needs to be run for each new graph.
    \item \textbf{Lacking Implicit Information and Strong Expressiveness:} $\bm{h}$ is fixed and not learnable. GNN \cite{kipf2016semi} shows that message passing can learn implicit information and provide better representations, which should be considered if possible.
\end{itemize}

\begin{figure*}[t] 
    \centering
    \includegraphics[width = 0.95\linewidth]{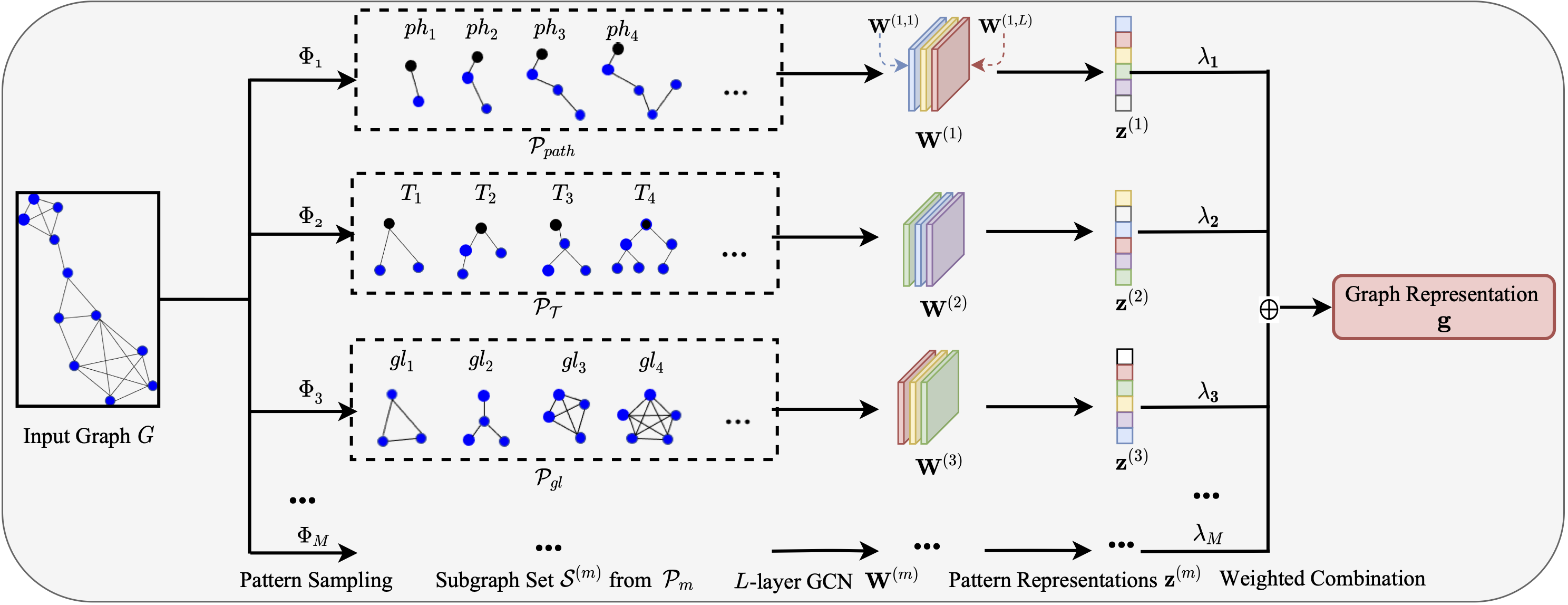}
    \caption{Framework of our proposed Pattern-based Explainable Graph Representation Learning with GNNs (\textbf{PXGL-GNN})}
    \label{fig:graphpatternframework}
\end{figure*}

\section{Learning Explainable Graph Representations via GNNs}\label{sec:PXGL-GNN}
In this section, we address the limitations pointed out in Section \ref{sec_limit} by proposing a GNN framework to learn and explain graph representations via pattern analysis. First, we sample graph substructures of various patterns from graph $G$. 
Given that overlaps may occur between patterns, we use the WL-test \cite{huang2021short} in each sampling phase to ensure that new samples are unique from existing ones.
The pattern sampling set $\mathcal{S}$ is defined as follows.
\begin{definition}[Pattern Sampling Set]\label{def:ps-set}
Let $S$ be a subgraph sampled from graph $G$. Given a graph pattern $\mathcal{P}$, the pattern sampling set $\mathcal{S}$ with $Q$ subgraphs of $G$ is defined as 
\begin{equation}\label{eqn:ps-set}
    \mathcal{S} := \{S_1, S_2, \ldots, S_q, \ldots, S_Q\},~~\text{where}~~S_q \in \mathcal{P}, ~ \forall~ q \in [Q]. 
\end{equation}
\end{definition}
Then, the pattern representation $\bm{z}$ is learned from the pattern sampling set as follows.
\begin{definition}[Pattern Representation]\label{def:pr-vec}
Given a graph $G$ and a pattern $\mathcal{P}$, we can obtain a pattern sampling set $\mathcal{S}$ using a sampling function $\Phi$. For each subgraph $S$ in the set $\mathcal{S}$, its adjacency matrix is $\bm{A}_S$ and its node feature matrix is $\bm{X}_S$. Let $F: \{0, 1\}^{|V_{S}|\times |V_{S}|} \times \mathbb{R}^{|V_{S}|\times d} \rightarrow \mathbb{R}^d$ be a pattern representation learning function with parameter $\mathcal{W}$, then the pattern representation $\bm{z} \in \mathbb{R}^d$ related to $G$ and $\mathcal{P}$ is defined as
\begin{equation}\label{eqn:pr-vec} 
    \bm{z} = \frac{1}{|\mathcal{S}|} \sum_{S \in \mathcal{S}} F(\bm{A}_S, \bm{X}_S; \mathcal{W}).
\end{equation}
\end{definition}
Finally, the ensemble representation $\bm{g}$ is the weighted sum of the $M$ pattern representations as follows.
\begin{definition}[Ensemble Representation]\label{def:ensemble-vec}
Given a set of patterns $\{ \mathcal{P}_1, \mathcal{P}_2, \ldots, \mathcal{P}_m, \ldots, \mathcal{P}_M\}$ and a graph $G$, the pattern sampling set $\mathcal{S}^{(m)}$ and the pattern representation $\bm{z}^{(m)}$ are related to the $m$-th pattern $\mathcal{P}_m$. The representation learning function $F(\cdot, \cdot; \mathcal{W}^{(m)})$ is used to learn $\bm{z}^{(m)}$ from $\mathcal{S}^{(m)}$. Let $\bm{\lambda} = [\lambda_1, \lambda_2, \ldots, \lambda_m, \ldots, \lambda_M]^\top$ be a positive weight parameter vector, then the ensemble representation $\bm{g} \in \mathbb{R}^d$ of graph $G$ is defined as
\begin{equation}\label{eqn:ensemble-vec} 
\begin{aligned}
    &\bm{g} = \sum_{m = 1}^M \lambda_m \bm{z}^{(m)}, ~~\text{with}~~ \\
    &\bm{z}^{(m)} = \frac{1}{|\mathcal{S}^{(m)}|} \sum_{S \in \mathcal{S}^{(m)}} F(\bm{A}_S, \bm{X}_S; \mathcal{W}^{(m)}), ~~\forall~ m \in [M]. 
\end{aligned}
\end{equation}
The weight parameter vector $\bm{\lambda}$ is constrained by $\bm{1}_M^\top \bm{\lambda} = 1$ and $\bm{\lambda} \ge 0$, and we can use the same softmax trick in computing the ensemble kernel \eqref{eqn:ensemble-ker} to remove this constraint. 
\end{definition}

Let $\mathbb{W} := \{\mathcal{W}^{(1)}, \mathcal{W}^{(2)}, \ldots, \mathcal{W}^{(m)}, \ldots, \mathcal{W}^{(M)} \}$ denote the trainable parameters of the GNN framework. To obtain the GNN parameters and the weight parameter $\bm{\lambda}$ in ensemble representation learning \eqref{eqn:ensemble-vec}, we define the supervised loss and unsupervised loss functions as follows.
\paragraph{Supervised Classification Loss.} Given a graph $G$, let $\bm{y} = [y_1, y_2, \ldots, y_c, \ldots, y_C]^\top \in \{0, 1\}^C$ be the ground truth label and  $\bm{\hat{y}} = [\hat{y}_1, \hat{y}_2, \ldots, \hat{y}_c, \ldots, \hat{y}_C]^\top \in \mathbb{R}^C$ be the predicted label. Let $f_c: \mathbb{R}^d \rightarrow \mathbb{R}^C$ be a classifier with softmax, i.e., $\bm{\hat{y}} = f_c(\bm{g})$, where the parameter is $\bm{W}_C$. Then the multi-class cross-entropy loss is defined as:  
\begin{equation}\label{eqn:loss-class}
    \begin{aligned}
            & \mathcal{L}_{\text{CE}}(\bm{\lambda}, \mathbb{W})=\frac{1}{|\mathcal{G}|}\sum_{G\in\mathcal{G}}\ell_{\text{CE}}(\bm{\lambda}, \mathbb{W};G) \\
            & =  - \frac{1}{|\mathcal{G}|}\sum_{G\in\mathcal{G}}\sum_{c = 1}^C y_c \log \hat{y}_c,  ~~\text{with}~~ \bm{\hat{y}} = f_c(\bm{g}).
    \end{aligned}
\end{equation}
\paragraph{Unsupervised KL Divergence.} We use the same KL divergence defined in Eq. \eqref{eqn:loss-kl-ker} as follows:
\begin{equation}\label{eqn:loss-kl-gau}
\begin{aligned}
        &\mathcal{L}_{\text{KL}}(\bm{K}(\bm{\lambda}, \mathbb{W})) = \mathbb{KL} (\bm{K}(\bm{\lambda}, \mathbb{W}), \bm{K}'(\bm{\lambda}, \mathbb{W})), \\
    &\text{with}~~K_{ij}(\bm{\lambda}, \mathbb{W}) = \exp\left(-\frac{\|\bm{g}_i - \bm{g}_j\|^2}{\gamma}\right),
\end{aligned}
\end{equation}
where $\mathcal{L}_{\text{KL}}(\bm{K}(\bm{\lambda}, \mathbb{W}))$ is a Gaussian kernel matrix of graph representations and  $\gamma$ is a positive parameter.

Finally, let $\mathcal{L}(\bm{\lambda}, \mathbb{W})$ be the supervised or unsupervised loss function when the graphs are labeled or unlabeled. The GNN parameters $\mathcal{W}$ and the weight parameters $\bm{\lambda}$ can be computed by solving
\begin{equation}
    \label{eqn:solve-ensemble-lambda}
    \begin{aligned}
    \bm{\lambda}^*, \mathbb{W}^* = \mathop{\textup{argmin}}{_{\mathbb{W}, \bm{1}_M^\top \bm{\lambda} = 1,~ \bm{\lambda} \ge 0}} ~~~\mathcal{L}(\bm{\lambda}, \mathbb{W}),
    \end{aligned}
\end{equation}
where $\bm{\lambda}^* = [\lambda_1^*, \ldots, \lambda_m^*, \ldots, \lambda_M^*]^\top$ and $\lambda_m^*$ indicates the contribution of the pattern representation $\bm{z}^{(m)}$ to the ensemble graph representation $\bm{g}$. For convenience, we call this method pattern-based XGL with GNNs, abbreviated as \textbf{PXGL-GNN}.

\section{Theoretical Analysis}
In this section, we provide a theoretical analysis of our method, focusing on robustness, generality, and complexity. We provide all detailed proof in the supplementary materials.

\subsection{Robustness Analysis}
Following \cite{o2021evaluation}, a learning method should be robust to small perturbations. Let $\Delta_A$ and $\Delta_X$ be perturbations on the adjacency matrix and node attributes. The perturbed graph is $\tilde{G} = (\bm{A} + \Delta_A, \bm{X} + \Delta_X)$, of which the representation is denoted as $\bm{\tilde{g}}$. We seek the upper bound of $\|\bm{\tilde{g}} - \bm{g}\|$.
Assume the representation learning function $F$ is an $L$-layer GCN \cite{kipf2016semi} with activation function $\sigma(\cdot)$ and average pooling as the output. For pattern $\mathcal{P}_m$, $F(\bm{A}, \bm{X}; \mathcal{W}^{(m)})$ has parameters $\mathcal{W}^{(m)} = \{\bm{W}^{(m, 1)}, \ldots, \bm{W}^{(m, L)}\}$, where $\bm{W}^{(m, L)}$ is the parameter in the $l$-th layer.

\begin{theorem}
Let $\bm{\tilde{A}} = \bm{A} + \Delta_A$ and $\bm{\tilde{X}} = \bm{X} + \Delta_X$. Suppose $\|\bm{A}\|_2 \leq \beta_A$, $\|\bm{X}\|_F \leq \beta_X$, $\|\bm{W}^{(m, l)}\|_2 \leq \beta_W$ for all $m \in [M]$ and $l \in [L]$, and $\sigma(\cdot)$ is $\rho$-Lipschitz continuous. Let $\alpha$ be the minimum node degree, and $\Delta_D := \bm{I} - \text{diag}(\bm{1}^\top (\bm{I} + \bm{A} + \Delta_A))^{\frac{1}{2}} \text{diag}(\bm{1}^\top \bm{A})^{-\frac{1}{2}}$. Then, 
\begin{equation*}
\begin{aligned}
        \|&\bm{\tilde{g}} - \bm{g}\|  \leq  \frac{1}{\sqrt{n}}  \rho^L \beta_W^L     (1 + \beta_A + \|\Delta_A \|_2)^{L-1} (1 + \alpha)^{-L} \\
        &\cdot\left[ (1 + \beta_A + 2 \|\Delta_A \|_2)\|\Delta_{X} \|_F + 2L\beta_X(1 + \beta_A) \|\Delta_D \|_2  \right]
\end{aligned}
\end{equation*}
\end{theorem}
The bound reveals that the method is sensitive to the perturbation on the graph structure, i.e., $\bm{A}$, when $L$ is large. It is relatively insensitive to the perturbation on $\bm{X}$. On the other hand,  when $\alpha$, the minimum node degree, is larger, the method is more robust.

\subsection{Generalization Analysis}
Following \cite{bousquet2002stability,feldman2019high}, we use uniform stability to derive the generalization bound for our model. Let $\bm{\lambda}$ and $\mathbb{W}$ be known parameters. The supervised loss $\ell_{\text{CE}}$ in Eq.\eqref{eqn:loss-class} is guaranteed with a uniform stability parameter $\eta$. The empirical risk $\mathcal{E}[\ell_{\text{CE}}(\bm{\lambda}, \mathbb{W};\mathcal{G})] := \frac{1}{N}\sum_{i=1}^N\ell_{\text{CE}}(\bm{\lambda}, \mathbb{W};G_i)$ and true risk $\mathbb{E}[\ell_{\text{CE}}(\bm{\lambda}, \mathbb{W};G)]$ have a high-probability generalization bound: for constant $c$ and $\delta \in (0, 1)$,
\begin{equation}\label{eqn:estimate-error}
\begin{aligned}
        \textbf{Pr} \Bigg[& |\mathbb{E}[\ell_{\text{CE}}(\bm{\lambda}, \mathbb{W};G) -  \mathcal{E}[\ell_{\text{CE}}(\bm{\lambda}, \mathbb{W};\mathcal{G})]| \geq \\
        &c \left(\eta \log(N) \log\left(\frac{N}{\delta}\right) + \sqrt{\frac{\log(1/\delta)}{N}} \right) \Bigg] \leq \delta.
\end{aligned}
\end{equation}
Let $\mathcal{D} := \{G_1, \ldots, G_{N}\}$ be the training data. By removing the $i$-th graph $G_i$, we get $\mathcal{D}^{\backslash i} = \{G_1, \ldots, G_{i-1}, G_{i+1}, \ldots, G_{N}\}$. Let $\bm{\lambda}_{\mathcal{D}}$ and $\mathcal{W}_{\mathcal{D}} := \{\bm{W}_C, \bm{W}_{\mathcal{D}}^{(m, l)}, \forall m \in [M], l \in [L]\}$ be the parameters trained on $\mathcal{D}$. Let $\bm{\lambda}_{\mathcal{D}^{\backslash i}}$ and $\mathcal{W}_{\mathcal{D}^{\backslash i}} := \{\bm{W}_{C^{\backslash i}}, \bm{W}_{\mathcal{D}^{\backslash i}}^{(m, l)}, \forall m \in [M], l \in [L]\}$ be the parameters trained on $\mathcal{D}^{\backslash i}$. We aim to find $\eta$ such that
\begin{equation}\label{eqn:uniform-stab}
     |\ell_{\text{CE}}(\bm{\lambda}_{\mathcal{D}}, \mathcal{W}_{\mathcal{D}}; G) - \ell_{\text{CE}}(\bm{\lambda}_{\mathcal{D}^{\backslash i}}, \mathcal{W}_{\mathcal{D}^{\backslash i}}; G)| \leq \eta
\end{equation}

\begin{theorem}
Suppose 
$$\max_{_{m\in[M],\,l\in[L]}}\Bigl\{ \|\bm{W}^{(m, l)}_{\mathcal{D}}\|_2,\quad \|\bm{W}^{(m, l)}_{\mathcal{D}^{\backslash i}}\|_2\Bigr\}\leq \hat{\beta}_W$$
$$\max_{m \in [M], l \in [L]} \|\bm{W}^{(m, l)}_{\mathcal{D}} -  \bm{W}^{(m, l)}_{\mathcal{D}^{\backslash i}}\|_2\leq \hat{\beta}_{\Delta W}$$
$$\| \bm{W}_C - \bm{W}_{C^{\backslash i}} \|_2\leq \gamma_{\Delta C},\quad \|\bm{W}_{C^{\backslash i}} \|_2\leq \gamma_C$$
Suppose the $f_c$ in $\ell_{\text{CE}}$ \eqref{eqn:loss-class} is a linear classifier, which is $\tau$-Lipschitz continuous. Suppose 
Thus the $\eta$ for estimation error \eqref{eqn:estimate-error} and uniform stability \eqref{eqn:uniform-stab} is:
\begin{equation}\label{eta}
\begin{aligned}
    \eta = &\frac{\tau}{\sqrt{n}} \rho^L \hat{\beta}_W^{L-1} \beta_X (1 + \beta_A)^L (1 + \alpha)^{-L} \\
    &\left[ \hat{\beta}_W \gamma_{\Delta C} + \gamma_{C} \left(2\hat{\beta}_W + L \hat{\beta}_{\Delta W}\right) \right]
\end{aligned}
\end{equation}
\end{theorem}
Invoking \eqref{eta} into \eqref{eqn:estimate-error}, we obtain the generalization error bound of our model. We see that when $\alpha$ is larger and $\beta_A,\beta_X$ are smaller, the generalization ability is stronger.

\subsection{Time and Space Complexity}

Given a dataset with $N$ graphs (each has $n$ nodes and $e$ edges), we select $M$ different patterns and sample $Q$ subgraphs of each pattern. First, our PXGL-EGK requires computing $M$ kernel matrices, of which the space complexity is $\mathcal{O}(MN^2)$, and the time complexity is related to those of different graph kernels. Assume $\psi_m$ is the time complexity of the $m$-th kernel, the total time complexity of PXGL-EGK is $\mathcal{O}(N^2\sum_{m=1}^M\psi_m)$. When $N$ is large, the method has high time and space complexities.

Regarding PXGL-GNN, suppose each representation learning function $F_m$ is an $L$-layer GCN, of which the width is linear with $d$. Let the batch size in the optimization be $B$ for both supervised and unsupervised learning. In supervised learning, the space complexity and time complexity of supervised learning are $\mathcal{O}(BMQ(e+nd)+MLd^2+Cd)$ and $\mathcal{O}(BMQL(ed+nd^2))$ respectively. In unsupervised learning, the space complexity and time complexity of supervised learning are $\mathcal{O}(BMQ(e+nd)+MLd^2+Cd+B^2)$ and $\mathcal{O}(BMQL(ed+nd^2)+B^2)$ respectively. This method is scalable to large graph datasets because the complexities are linear with $BMQ$ and $B^2$, where the $B^2$ term, referring to Eq.~\eqref{eqn:loss-kl-gau}, comes from computing the Gaussian kernel matrix between all pairs of examples in a batch.

\section{Related Works}
Due to space limitations, we introduce previous works on explainable graph learning (XGL), graph representation learning (GCL), and graph kernels in the supplementary materials. 

\section{Experiments}
\begin{table}[t]
      \centering
      \small % set table font size in 9pt as required, equal to \fontsize{9pt}{11pt}\selectfont
      \resizebox{0.45\textwidth}{!}{
      \renewcommand{\arraystretch}{1}
      \begin{tabular}{c|c|c|c|c|c}
        \toprule
         Name & \makecell{\# of \\graphs} & \makecell{\# of \\classes} & \makecell{\# of \\nodes } & \makecell{node \\labels} & \makecell{node\\ attributes} \\ \midrule
         MUTAG & 188 & 2  & 17.9  & yes & no\\ 
         PROTEINS &  1113 & 2  & 39.1  & yes & yes\\
         DD &1178 &2 &284.32 & yes & no \\
         NCI1 & 4110 & 2 & 29.9  & yes  &  no  \\  
         COLLAB  &5000 &    3 &74.49 & no & no \\
         IMDB-B & 1000 & 2  & 19.8  & no  & no\\
         REDDIT-B & 2000 & 2 & 429.63 &no &no \\
         REDDIT-M5K & 4999 & 5 & 508.52 &no  &no \\ 
         \bottomrule
      \end{tabular}}
      \caption{Statistics of Datasets}
      \label{tab:TUdataset}
\end{table}

We test our method on the TUdataset \cite{Morris+2020} for both supervised and unsupervised learning tasks, as shown in Table \ref{tab:TUdataset}. Our goal is to learn explainable graph representations. We provide the weight parameter $\bm{\lambda}$ and visualize the ensemble representation $\bm{g}$ and the pattern representation $\bm{z}^{(m)}$. We use seven graph patterns: paths, trees, graphlets, cycles, cliques, wheels, and stars, sampling $Q = 10$ subgraphs for each. We select these patterns based on their discriminative power and computational feasibility. In practice, one could use a subset of these seven patterns and adjust the sampling cardinality $Q$ based on domain knowledge or computational constraints. We use a 5-layer GCN for the representation learning function $F$ and a 3-layer DNN with softmax for the classification function $f_c$. Experiments are repeated ten times and the average value and standard deviation are reported. Due to the space limitation, the results of PXGL-EGK and other details are shown in the supplementary materials.

\begin{figure*}[ht]
    \centering
    \begin{subfigure}[b]{0.2\textwidth}
        \centering
        \caption{$\bm{g}$: ensemble}
        \includegraphics[width=\textwidth]{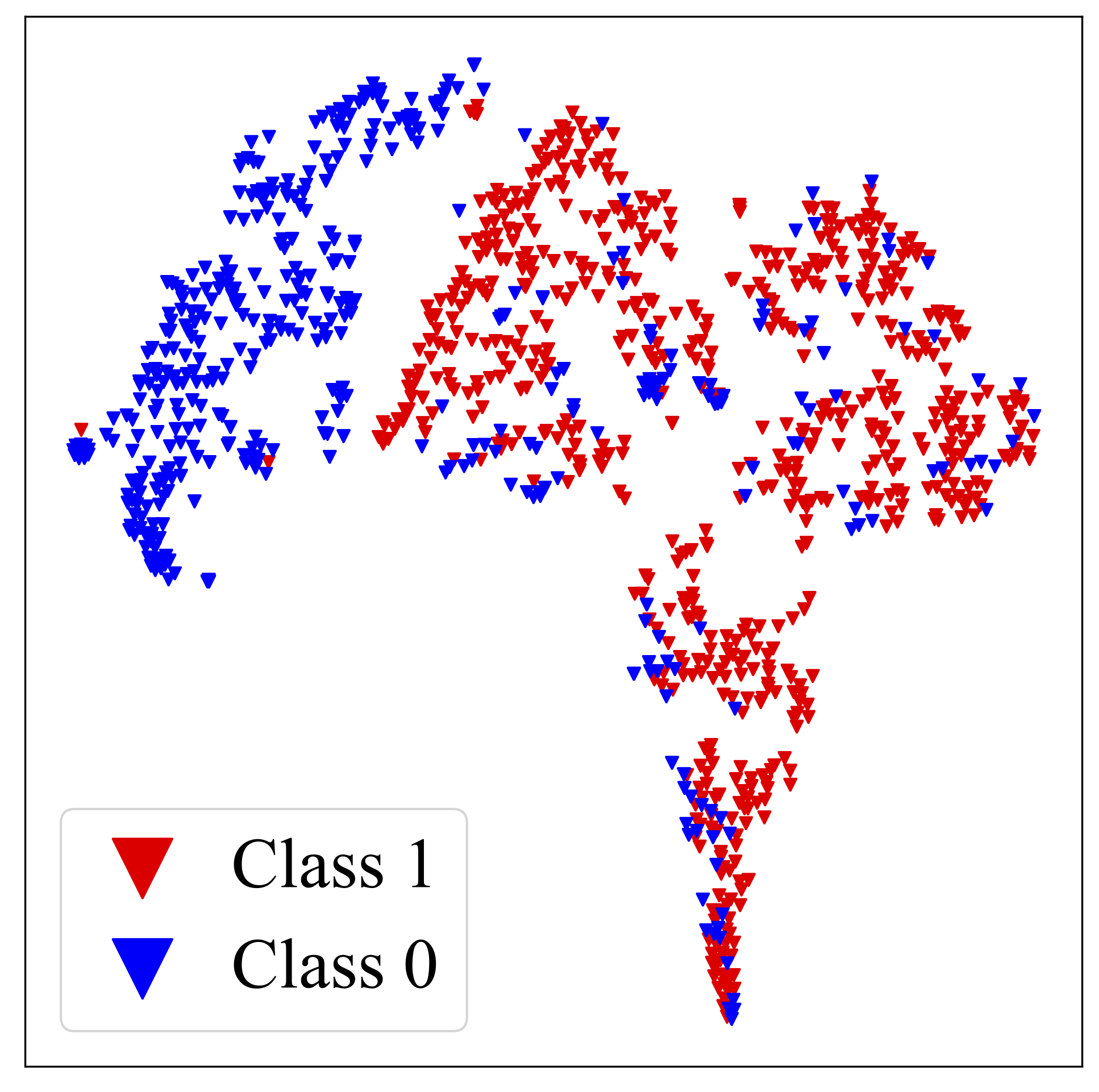}\\
        \footnotesize{$ \sum_{m = 1}^M \lambda_m \bm{z}^{(m)}$}
        \label{fig:GNN_sup_K_combined}
    \end{subfigure}
    \hfill
    \begin{subfigure}[b]{0.2\textwidth}
        \centering
        \caption{$\bm{z}^{(1)}$: path}
        \includegraphics[width=\textwidth]{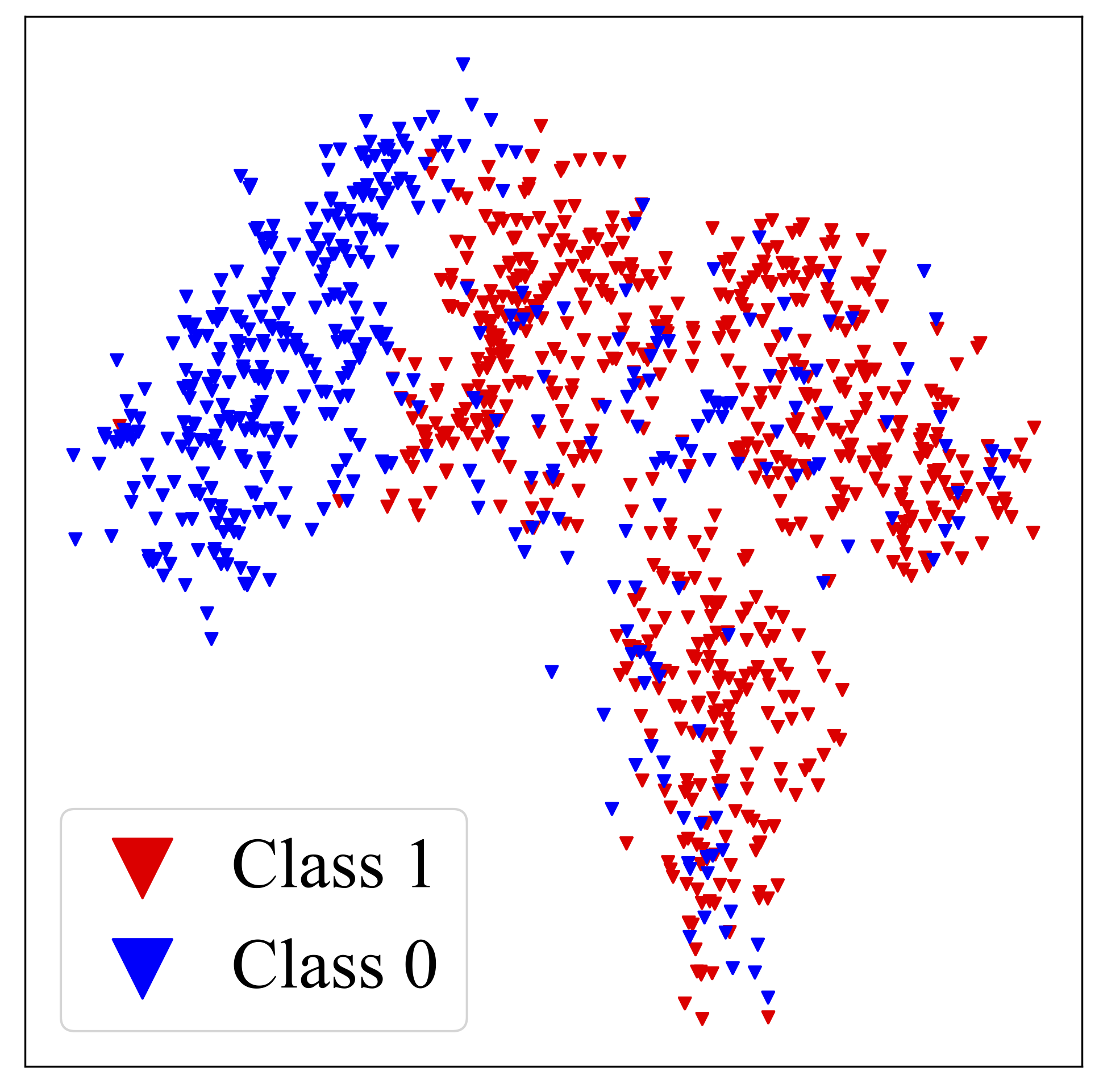}\\
        \footnotesize{$(\lambda_1 = 0.5504)$}
        \label{fig:GNN_sup_PROTEINS_path}
    \end{subfigure}
    \hfill
    \begin{subfigure}[b]{0.2\textwidth}
        \centering
        \caption{$\bm{z}^{(2)}$: tree}
        \includegraphics[width=\textwidth]{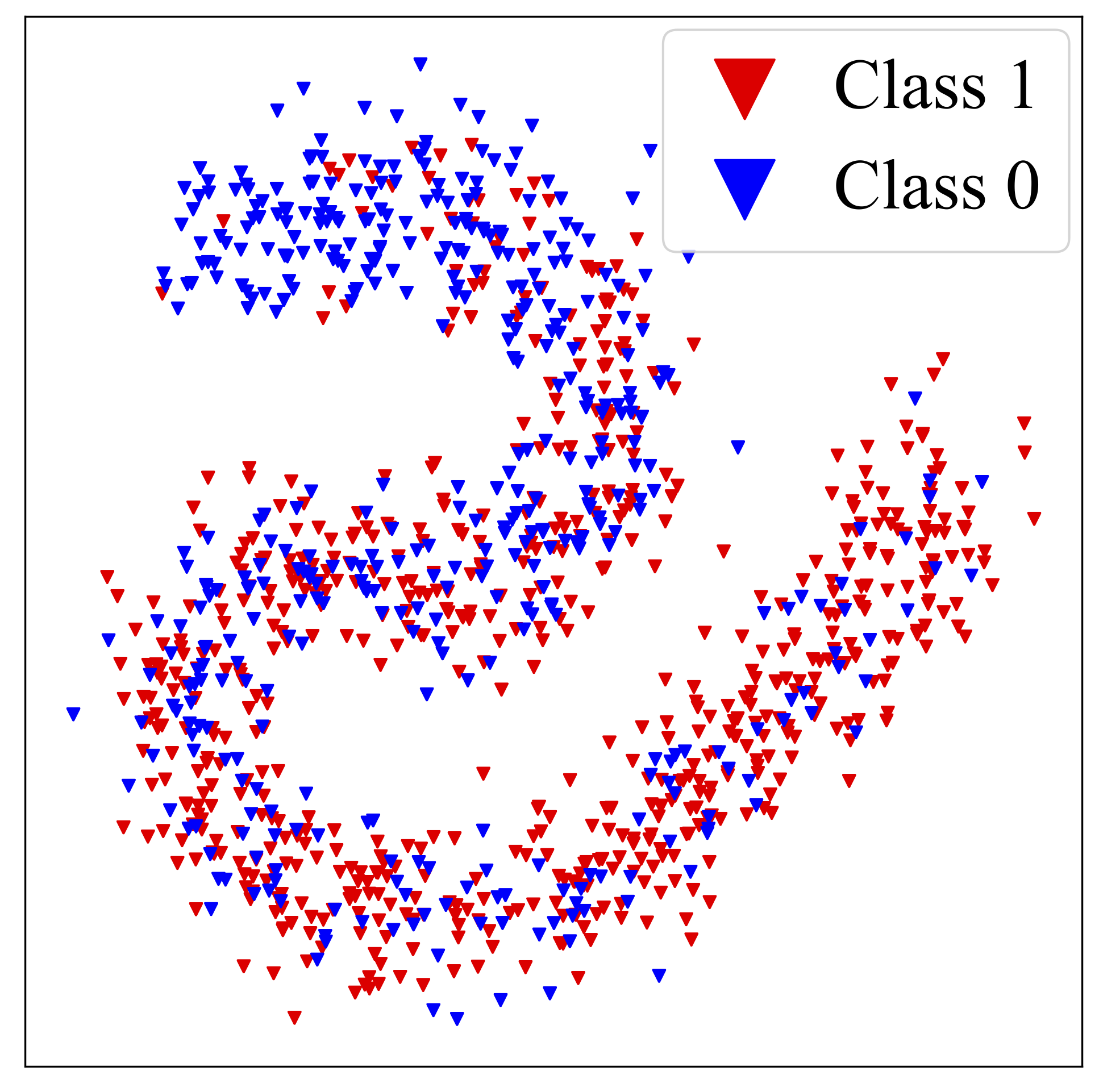}
        \footnotesize{$(\lambda_2 = 0.0746)$}
        \label{fig:GNN_sup_PROTEINS_subtree}
    \end{subfigure}
    \hfill
    \begin{subfigure}[b]{0.2\textwidth}
        \centering
    \caption{$\bm{z}^{(3)}$: graphlet}
        \includegraphics[width=\textwidth]{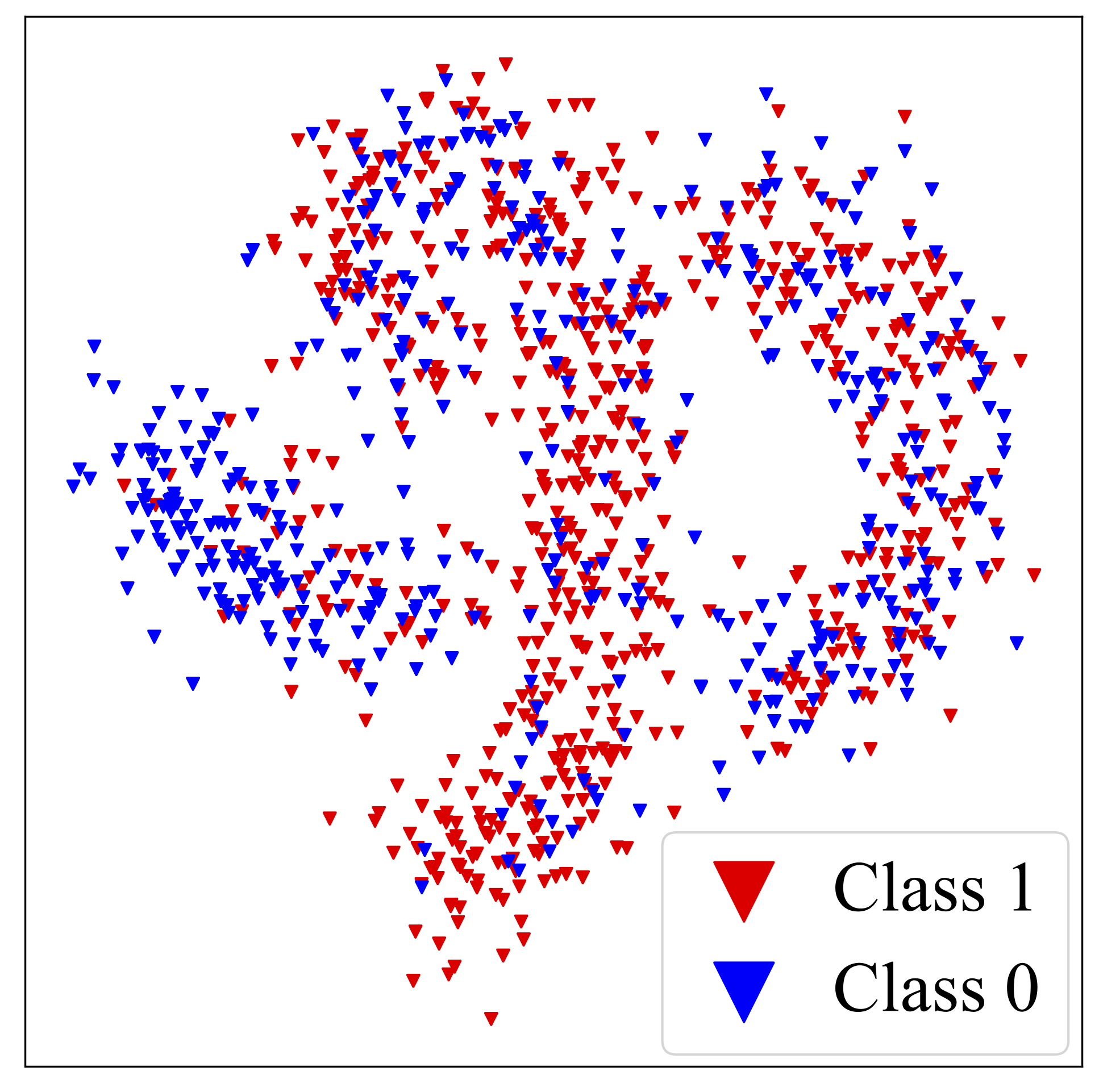}\\
        \footnotesize{$(\lambda_3 = 0.08103)$}\label{fig:GNN_sup_PROTEINS_graphlet}
    \end{subfigure}
    \caption{t-SNE visualizations of PXGL-GNN's pattern representations (supervised) for PROTEINS.}
    \label{fig:GNN_sup_pattern-protein}
\end{figure*}

\begin{table*}[h!]
\centering
\small % set table font size in 9pt as required, equal to \fontsize{9pt}{11pt}\selectfont
\resizebox{1\linewidth}{!}{
\renewcommand{\arraystretch}{0.85}
\begin{tabular}{c|cccccccc}
\toprule
Pattern     & MUTAG & PROTEINS & DD & NCI1 & COLLAB & IMDB-B & REDDIT-B & REDDIT-M5K \\ \midrule
paths       & $\color{blue}0.095 \pm 0.014$ & $\bf0.550 \pm 0.070$ & $0.093 \pm 0.012$ & $0.022 \pm 0.002$ & $\bf0.587 \pm 0.065$ & $\color{blue}0.145 \pm 0.018$ & $0.131 \pm 0.027$ & $0.027 \pm 0.003$ \\
trees       & $0.046 \pm 0.005$ & $0.074 \pm 0.009$ & $0.054 \pm 0.006$ & $0.063 \pm 0.008$ & $0.105 \pm 0.013$ & $0.022 \pm 0.003$ & $0.055 \pm 0.007$ & $0.025 \pm 0.003$ \\
graphlets   & $0.062 \pm 0.008$ & $0.081 \pm 0.011$ & $\color{blue}0.125 \pm 0.015$ & $0.101 \pm 0.013$ & $0.063 \pm 0.008$ & $0.084 \pm 0.011$ & $0.026 \pm 0.003$ & $0.054 \pm 0.007$ \\
cycles      & $\bf 0.654 \pm 0.085$ & $0.099 \pm 0.013$ & $0.094 \pm 0.012$ & $\color{blue}0.176 \pm 0.022$ & $0.022 \pm 0.003$ & $0.123 \pm 0.016$ & $0.039 \pm 0.005$ & $0.037 \pm 0.005$ \\
cliques     & $0.082 \pm 0.011$ & $\color{blue}0.098 \pm 0.012$ & $\bf0.572 \pm 0.073$ & $\bf0.574 \pm 0.075$ & $\color{blue}0.134 \pm 0.017$ & $\bf0.453 \pm 0.054$ & $\color{blue} 0.279 \pm 0.069$ & $\color{blue}0.256 \pm 0.067$ \\
wheels      & $0.026 \pm 0.003$ & $0.039 \pm 0.005$ & $0.051 \pm 0.007$ & $0.012 \pm 0.002$ & $0.068 \pm 0.009$ & $0.037 \pm 0.004$ & $0.036 \pm 0.005$ & $0.023 \pm 0.003$ \\
stars       & $0.035 \pm 0.005$ & $0.056 \pm 0.007$ & $0.011 \pm 0.002$ & $0.052 \pm 0.007$ & $0.021 \pm 0.003$ & $0.136 \pm 0.017$ & $\bf 0.447 \pm 0.006$ & $\bf 0.578 \pm 0.033$ \\ 
\bottomrule
\end{tabular}}
\caption{The learned $\bm{\lambda}$ of PXGL-GNN (supervised). The largest value is {\bf bold} and the second largest value is {\color{blue} blue}.}
\label{tab:lambda-supervised}
\end{table*}
\begin{table*}[h!]
\centering
\small % set table font size in 9pt as required, equal to \fontsize{9pt}{11pt}\selectfont
\resizebox{0.85\linewidth}{!}{
\renewcommand{\arraystretch}{0.85}
\begin{tabular}{c|cccccccc}
\toprule
Method    & MUTAG                 & PROTEINS         & DD               & NCI1             & COLLAB           & IMDB-B           & REDDIT-B         & REDDIT-M5K  \\ \midrule
GIN        & 84.53 $\pm$ 2.38     & 73.38 $\pm$2.16  & 76.38 $\pm$1.58  & 73.36 $\pm$1.78  & 75.83 $\pm$ 1.29 & 72.52 $\pm$ 1.62 & 83.27 $\pm$ 1.30 & 52.48 $\pm$  1.57\\
DiffPool   & 86.72 $\pm$ 1.95    & 76.07 $\pm$1.62  & 77.42 $\pm$2.14  & 75.42 $\pm$2.16  & 78.77 $\pm$ 1.36 & 73.55 $\pm$ 2.14  & 84.16 $\pm$ 1.28 & 51.39 $\pm$  1.48\\
DGCNN      & 84.29 $\pm$ 1.16   & 75.53 $\pm$2.14  & 76.57 $\pm$1.09  & 74.81 $\pm$1.53  & 77.59 $\pm$ 2.24 & 72.19 $\pm$ 1.97 & 86.33 $\pm$ 2.29 & 53.18 $\pm$  2.41\\
GRAPHSAGE  & 86.35 $\pm$ 1.31   & 74.21 $\pm$1.85  & 79.24 $\pm$2.25  & 77.93 $\pm$2.04  & 76.37 $\pm$ 2.11 & 73.86 $\pm$ 2.17  & 85.59 $\pm$ 1.92 & 51.65 $\pm$  2.55\\
SubGNN     & 87.52 $\pm$ 2.37   & 76.38 $\pm$1.57  & 82.51 $\pm$1.67  & 82.58 $\pm$1.79  & 81.26 $\pm$ 1.53 & 71.58 $\pm$ 1.20  & 88.47 $\pm$ 1.83 & 53.27 $\pm$  1.93\\
SAN        & 92.65 $\pm$ 1.53   & 75.62 $\pm$2.39  & 81.36 $\pm$2.10  &\color{blue} 83.07 $\pm$1.54  &\textbf{\color{blue}} 82.73 $\pm$ 1.92 & 75.27 $\pm$ 1.43  & 90.38 $\pm$ 1.54 & 55.49 $\pm$  1.75\\
SAGNN      &\color{blue} 93.24 $\pm$ 2.51   & 75.61 $\pm$2.28  & 84.12 $\pm$1.73  & 81.29 $\pm$1.22  & 79.94 $\pm$ 1.83 & 74.53 $\pm$ 2.57 & 89.57 $\pm$ 2.13 & 54.11 $\pm$  1.22\\
ICL        & 91.34 $\pm$ 2.19   & 75.44 $\pm$1.26  & 82.77 $\pm$1.42  & 83.45 $\pm$1.78  & 81.45 $\pm$ 1.21 & 73.29 $\pm$ 1.46   &\color{blue} 90.13 $\pm$ 1.40 &\color{blue} 56.21 $\pm$  1.35\\
S2GAE      & 89.27 $\pm$ 1.53   &\color{blue} 76.47 $\pm$1.12  &\color{blue} 84.30 $\pm$1.77  & 82.37 $\pm$2.24  & 82.35 $\pm$ 2.34 &\color{blue} 75.77 $\pm$ 1.72   & 90.21 $\pm$ 1.52 & 54.53 $\pm$  2.17\\ \hline
\addlinespace[2pt]
\textbf{PXGL-GNN}      &\bf 94.87 $\pm$ 2.26   &\bf 78.23 $\pm$2.46  &\bf 86.54 $\pm$1.95  &\bf 85.78 $\pm$2.07  &\bf 83.96 $\pm$ 1.59 &\bf 77.35 $\pm$ 2.32   &\bf 91.84 $\pm$ 1.69 &\bf 57.36 $\pm$  2.14 \\ \bottomrule
\end{tabular}}
\caption{Accuracy (\%) of Graph Classification. The best accuracy is {\bf bold} and the second best is {\color{blue} blue}.}
\label{tab:supervised-acc}
\end{table*}

\subsection{Supervised Learning}
We conduct supervised XGL via pattern analysis, the proposed PXGL-GNN, by solving optimization \eqref{eqn:solve-ensemble-lambda} with the classification loss \eqref{eqn:loss-class}. The dataset is split into 80\% training, 10\% validation, and 10\% testing data. The learned weight parameter $\bm{\lambda}$, indicating each pattern's contribution to graph representation learning, is reported in Table \ref{tab:lambda-supervised}. We also visualize the graph representation $\bm{g}$ and three pattern representations $\bm{z}^{(m)}$ of PROTEINS in Figure~\ref{fig:GNN_sup_pattern-protein}. Results show the paths pattern is most important for learning $\bm{g}$, and the ensemble representation $\bm{g}$ outperforms single pattern representations $\bm{z}^{(m)}$, which reveal underlying structural characteristics and naturally align with domain knowledge since paths are crucial for reflecting protein folding pathways~\cite{yan2011applications}.

The compared baselines include classical GNNs like GIN \cite{xu2018powerful}, DiffPool \cite{ying2018hierarchical}, DGCNN \cite{zhang2018end}, GRAPHSAGE \cite{hamilton2017inductive}; subgraph-based GNNs like SubGNN \cite{kriege2012subgraph}, SAN \cite{zhao2018substructure}, SAGNN \cite{zeng2023substructure}; and recent methods like S2GAE \cite{tan2023s2gae} and ICL \cite{zhao2024twist}. The accuracies in Table \ref{tab:supervised-acc} show that our method performs the best.

\subsection{Unsupervised Learning}

We conduct unsupervised XGL via pattern analysis, the proposed PXGL-GNN, by solving optimization \eqref{eqn:solve-ensemble-lambda} with the KL divergence loss \eqref{eqn:loss-kl-gau}.
The learned weight parameter $\bm{\lambda}$ for XGL is reported in the supplementary materials. The visualizations of unsupervised XGL results are in the supplementary materials.
Results show that the ensemble representation $\bm{g}$ outperforms single pattern representations $\bm{z}^{(m)}$.

For clustering performance, we use clustering accuracy (ACC) and Normalized Mutual Information (NMI). Baselines include four kernels: Random walk kernel (RW) \cite{borgwardt2005protein}, Sub-tree kernels \cite{da2012tree}, Graphlet kernels \cite{prvzulj2007biological}, Weisfeiler-Lehman (WL) kernels \cite{kriege2012subgraph}; and three unsupervised graph representation learning methods with Gaussian kernel in Eq.~\eqref{eqn:loss-kl-gau}: InfoGraph \cite{sun2019infograph}, GCL \cite{you2020graph}, 
GraphACL \cite{luo2023self}. The results are in Table~4 in our supplementary materials, and Table~5 reports the performance of PXGL-EGK. Our methods outperformed all benchmarks in almost all cases.

\section{Conclusion}
This paper investigates the explainability of graph representations through two novel approaches. First, we develop \textbf{PXGL-EGK} based on graph ensemble kernels that captures structural similarities while maintaining interpretability. Second, we introduce \textbf{PXGL-GNN}, a framework that incorporates diverse graph patterns (paths, trees, etc.) into GNNs to enhance both performance and explainability. We establish theoretical guarantees for our proposed methods, including robustness certification against perturbations and non-asymptotic generalization bounds. Extensive empirical evaluation demonstrates that our approaches not only achieve superior performance on classification and clustering tasks across multiple datasets, but also provide interpretable explanations for the learned graph representations.

\section*{Acknowledgements}
This work was supported by the National Natural Science Foundation of China under
Grant No.62376236, the Guangdong Provincial Key Laboratory of Mathematical Foundations for Artificial Intelligence (2023B1212010001), Shenzhen Science and Technology Program ZDSYS20230626091302006 (Shenzhen Key Lab of Multi-Modal Cognitive Computing), and Shenzhen Stability Science Program 2023. 

\section*{Contribution Statement}
Xudong Wang and Ziheng Sun contributed equally.
}
\bibliographystyle{named}
\bibliography{ijcai25}

\newpage
\onecolumn
\appendix

\section{Notations and Definitions of Patterns}\label{app:def-pattern}
In our work, graph patterns are referred to as subgraphs with practical meanings. 
Let $G = (V, E)$ be a graph. A subgraph $S = (V_S, E_S)$ of $G$ is defined such that $V_S \subseteq V$ and $E_S \subseteq E \cap (V_S \times V_S)$. The mathematical definitions of graph patterns are as follows:

\begin{itemize}
    \item \textbf{Paths:}
    $S$ is a \textit{path} if there exists a sequence of distinct vertices $v_1, \ldots, v_k \in V_S$ such that $E_S = ((v_i, v_{i+1}) : i = 1, \ldots, k-1)$.

    \item \textbf{Trees:}
    $S$ is a \textit{tree} if it is connected and contains no cycles, i.e., it is acyclic and $|E_S| = |V_S| - 1$.

    \item \textbf{Graphlets:}
    $S$ is a \textit{graphlet} if it is a small connected induced subgraph of $G$, typically consisting of 2 to 5 vertices.

    \item \textbf{Cycles:}
    $S$ is a \textit{cycle} if there exists a sequence of distinct vertices $v_1, \ldots, v_k \in V_S$ such that $E_S = ((v_i, v_{i+1}) : i = 1, \ldots, k-1) \cup ((v_k, v_1))$.

    \item \textbf{Cliques:}
    $S$ is a \textit{clique} if every two distinct vertices in $V_S$ are adjacent, thus $E_S = ((v_i, v_j) : v_i, v_j \in V_S, i \neq j)$.

    \item \textbf{Wheels:}
    $S$ is a \textit{wheel} if it consists of a cycle with vertices $v_1, \ldots, v_{k-1}$ and an additional central vertex $v_k$ such that $v_k$ is connected to all vertices of the cycle.

    \item \textbf{Stars:}
    $S$ is a \textit{star} if it consists of one central vertex $v_c$ and several leaf vertices $v_1, \ldots, v_{k-1}$, where each leaf vertex is only connected to $v_c$. Thus, $E_S = ((v_c, v_i) : i = 1, \ldots, k-1)$.
\end{itemize}

\section{Related Works} \label{app:related-work}
In this section, we introduce previous works on explainable graph learning (XGL), graph representation learning (GRL), and graph kernels.

\subsection{Explainable Graph Learning (XGL)}\label{app:XGL}
Explainable artificial intelligence (XAI) is a rapidly growing area in the AI community \cite{dovsilovic2018explainable,adadi2018peeking,angelov2021explainable,hassija2024interpreting}. Explainable graph learning (XGL) \cite{kosan2023gnnx} can be generally classified into two categories: model-level methods and instance-level methods.

\paragraph{Model-level}
Model-level or global explanations aim to understand the overall behavior of a model by identifying patterns in its predictions.
For examples, XGNN\cite{yuan2020xgnn} trains a graph generator to create graph patterns that maximize a certain prediction, providing high-level insights into GNN behavior. GLG-Explainer\cite{azzolin2022global} combines local explanations into a logical formula over graphical concepts, offering human-interpretable global explanations aligned with ground-truth or domain knowledge. GCFExplainer\cite{huang2023global} uses global counterfactual reasoning to find representative counterfactual graphs, providing a summary of global explanations through vertex-reinforced random walks on an edit map of graphs.

\paragraph{Instance-level}
Instance-level methods offer explanations tailored to specific predictions, focusing on why particular instances are classified in a certain manner. For instance, GNNExplainer \cite{ying2019gnnexplainer} identifies a compact subgraph structure and a small subset of node features crucial for a GNN's prediction. PGExplainer \cite{luo2020parameterized} trains a graph generator to incorporate global information and uses a deep neural network (DNN) to parameterize the explanation generation process. SubgraphX \cite{yuan2021explainability} efficiently explores different subgraphs using Monte Carlo tree search to explain predictions. RG-Explainer \cite{shan2021reinforcement} constructs a connected explanatory subgraph by sequentially adding nodes, consistent with the message passing scheme. MixupExplainer \cite{zhang2023mixupexplainer} introduces a general form of Graph Information Bottleneck (GIB) to address distribution shifting issues in post-hoc graph explanation. AutoGR \cite{wang2021explainable} introduces an explainable AutoML approach for graph representation learning.

\subsection{Graph Representation Learning}\label{app:GRL}
Graph representation learning is crucial for transforming complex graphs into vectors, particularly for tasks like classification. The methods for graph representation learning are mainly classified into two categories: supervised and unsupervised learning.

\paragraph{Supervised Representation Learning}
Most GNNs can be used in supervised graph representation learning tasks by aggregating all the node embeddings into a graph representation using a readout function \cite{hamilton2020graph,chami2022machine}. Besides traditional GNNs like GCN \cite{kipf2016semi}, GIN \cite{xu2018powerful}, and GAT \cite{velivckovic2017graph}, recent works include:
Template-based Fused Gromov-Wasserstein (FGW) \cite{vincent2022template} computes a vector of FGW distances to learnable graph templates, acting as an alternative to global pooling layers.
Path Isomorphism Network (PIN) \cite{truong2024weisfeiler} introduces a graph isomorphism test and a topological message-passing scheme operating on path complexes.
Graph U-Net \cite{amouzad2024graph} proposes GIUNet for graph classification, combining node features and graph structure information using a pqPooling layer.
Unified Graph Transformer Networks (UGT) \cite{lee2024transitivity} integrate local and global structural information into fixed-length vector representations using self-attention.
CIN++ \cite{giusti2023cin++} enhances topological message passing to account for higher-order and long-range interactions, achieving state-of-the-art results.
Graph Joint-Embedding Predictive Architectures (Graph-JEPA) \cite{skenderi2023graph} use masked modeling to learn embeddings for subgraphs and predict their coordinates on the unit hyperbola in the 2D plane.

\paragraph{Unsupervised Representation Learning}
Unsupervised methods aim to learn graph representations without labeled data. Notable methodologies include:
InfoGraph \cite{sun2019infograph} emphasizes mutual information between graph-level and node-level representations.
Graph Contrastive Learning techniques \cite{you2020graph,suresh2021adversarial,you2021graph} enhance graph representations through diverse augmentation strategies.
AutoGCL \cite{yin2022autogcl} introduces learnable graph view generators.
GraphACL \cite{luo2023self} adopts a novel self-supervised approach.
InfoGCL \cite{xu2021infogcl} and SFA \cite{zhang2023spectral} focus on information transfer and feature augmentation in contrastive learning.
Techniques like GCS \cite{wei2023boosting}, NCLA \cite{shen2023neighbor}, S$^3$-CL \cite{ding2023eliciting}, and ImGCL \cite{zeng2023imgcl} refine graph augmentation and learning methods.
GRADATE \cite{duan2023graph} integrates subgraph contrast into multi-scale learning networks. GeMax \cite{pmlr-v235-sun24i} explores the unsupervised representation learning via graph entropy maximization. 

\paragraph{GNNs using Subgraphs and Substructures}
Our pattern analysis method samples subgraphs from different graph patterns to conduct explainable graph representation learning. The key novelty and contribution of our paper is that graph pattern analysis provides explainability for representations. We discuss other GNN methods based on subgraphs and substructures here:
Subgraph Neural Networks (SubGNN) \cite{kriege2012subgraph} learn disentangled subgraph representations using a novel subgraph routing mechanism, but they sample subgraphs randomly, lacking explainability.
Substructure Aware Graph Neural Networks (SAGNN) \cite{zeng2023substructure} use cut subgraphs and return probability to capture structural information but focus on expressiveness rather than explainability.
Mutual Information (MI) Induced Substructure-aware GRL \cite{wang2020exploiting} maximizes MI between original and learned representations at both node and graph levels but does not provide explainable representation learning.
Substructure Assembling Network (SAN) \cite{zhao2018substructure} hierarchically assembles graph components using an RNN variant but lacks explainability in representation learning.

Several works focus on analyzing the expressiveness of methods by their ability to count substructures, but they do not provide explainable representation learning. For example:
\cite{chen2020can} analyze the expressiveness of MPNNs \cite{gilmer2017neural} and 2nd-order Invariant Graph Networks (2-IGNs) \cite{maron2019universality} based on their ability to count specific subgraphs, highlighting tasks that are challenging for classical GNN architectures but not focusing on explainability.
\cite{frasca2022understanding} compare the expressiveness of SubGNN \cite{kriege2012subgraph} and 2-IGNs \cite{maron2019universality} using symmetry analysis, establishing a link between Subgraph GNNs and Invariant Graph Networks.

\subsection{Graph Kernels}
Graph kernels evaluate the similarity between two graphs. Over the past decades, numerous graph kernels have been proposed \cite{siglidis2020grakel}. We classify them into two categories: pattern counting kernels and non-pattern counting kernels.

\paragraph{Pattern Counting Kernels}
Pattern counting kernels compare specific substructures within graphs to evaluate similarity \cite{kriege2020survey}. For examples,
Random walk kernels \cite{borgwardt2005protein,gartner2003graph} measure graph similarity by counting common random walks between graphs.
Shortest-path kernels \cite{borgwardt2005shortest} compare graphs using the shortest distance matrix generated by the Floyd-Warshall algorithm, based on edge values and node labels.
Sub-tree kernels \cite{da2012tree,smola2002fast} decompose graphs into ordered Directed Acyclic Graphs (DAGs) and use tree kernels extended to DAGs.
Graphlet kernels \cite{prvzulj2007biological} count small connected non-isomorphic subgraphs (graphlets) within graphs and compare their distributions.
Weisfeiler-Lehman subtree kernels \cite{kriege2012subgraph} use small subgraphs, like graphlets, to compare graphs, allowing flexibility to compare vertex and edge attributes with arbitrary kernel functions.

\paragraph{Non-pattern Counting Kernels}
Non-pattern counting kernels evaluate graph similarity without relying on specific substructure counts. For examples,
Neighborhood hash kernel \cite{hido2009linear} use binary arrays to represent node labels and logical operations on connected node labels. This kernel has linear time complexity.
GraphHopper kernel \cite{feragen2013scalable} compare shortest paths between node pairs using kernels on nodes encountered while hopping along shortest paths.
Graph hash kernel \cite{shi2009hash} use hashing for efficient kernel computation, suitable for data streams and sparse feature spaces, with deviation bounds from the exact kernel matrix.
Multiscale Laplacian Graph (MLG) kernel \cite{kondor2016multiscale} account for structure at different scales using Feature Space Laplacian Graph (FLG) kernels, applied recursively to subgraphs. They introduce a randomized projection procedure similar to the Nystrom method for RKHS operators.

\section{Proof for Robustness Analysis}\label{app:robustness}
Let $\Delta_A$ and $\Delta_X$ be some perturbations on adjacency matrix and node attributes, then the perturbed graph is denoted as $\tilde{G} = (\bm{A} + \Delta_A, \bm{X} + \Delta_X)$. 
Let $\bm{g}$ be the graph representation of $G$ and $\bm{\tilde{g}}$ be the graph representation of $\tilde{G}$. The robustness analysis is to find the upper bound of $\|\bm{\tilde{g}} - \bm{g} \|$. 

\paragraph{Assumptions and Notations:} Let $\bm{\tilde{A}} = \bm{A} + \Delta_A$ and $\bm{\tilde{X}} = \bm{X} + \Delta_X$. We suppose that $\|\bm{A}\|_2 \leq \beta_A$, $\|\bm{X}\|_F \leq \beta_B$ and $\|\bm{W}^{(m, l)}\|_2 \leq \beta_W,~(\forall~m \in [M],~l \in [L])$, the activation $\sigma(\cdot)$ of GCN is $\rho$-Lipschitz continuous. We denote the minimum node degree of $G$ as $\alpha$,  the effects of structural perturbation as $\kappa = \min(\bm{1}^\top \Delta_A)$, and $\Delta_D := \bm{I} - \text{diag}(\bm{1}^\top (\bm{I} + \bm{A} + \Delta_A))^{\frac{1}{2}} \text{diag}(\bm{1}^\top \bm{A})^{-\frac{1}{2}}$.

\paragraph{Theorem:} Our conclusion for robustness analysis is as follows:
\begin{equation}
    \begin{aligned}
        \|\bm{\tilde{g}}- \bm{g}  \|  &\leq \frac{1}{\sqrt{n}} \rho^{L} \beta_W^{L} \beta_X (1 + \alpha)^{-L} (1 + \beta_A + \|\Delta_A \|_2)^L\\
     &\Bigl( 1 + 2L  \|\Delta_D \|_2 + L (1 + \beta_A + \|\Delta_A \|_2)^{-1}  \| \Delta_A \|_2) \Bigr)
    \end{aligned}
\end{equation}

To provide a clearer analysis, we first use the whole graph $G$ and $\tilde{G}$ as the input of the pattern representation learning function $F$ without sampling the subgraphs. Then we consider using the subgraph sampling to analyze $\bm{g}$ and $\bm{\tilde{g}}$ and finally finish the proof of robustness analysis. 

\subsection{Learning Pattern Representations using the Whole Graph without Sampling}
In this section, we first consider using the whole graph $G$ and $\tilde{G}$ as the input of the pattern representation learning function $F$ without sampling the subgraphs, i.e., we analyze $F(\bm{A}, \bm{X}; \mathcal{W}^{(m)})$ and $F(\bm{\tilde{A}}, \bm{\tilde{X}}; \mathcal{W}^{(m)})$.

\paragraph{Representation Learning Function $F$}
In theoretical analysis, we suppose the pattern representation learning function $F$ is a $L$-layer GCN \cite{kipf2016semi} with an average pooling $\text{avg-pool}: \mathbb{R}^{n \times d} \rightarrow \mathbb{R}^d$ as the output layer. The pattern learning function for the pattern $\mathcal{P}_m$ is denoted as $F(\bm{A}, \bm{X}; \mathcal{W}^{(m)})$, where $\mathcal{W}^{(m)} = \{\bm{W}^{(m, 1)},..., \bm{W}^{(m, l)}, ..., \bm{W}^{(m, L)}\}$ and $\bm{W}^{(m, l)}$ is the 
trainable parameter of the $l$-th layer. We use the adjacency matrix $\bm{A}$ and node feature matrix $\bm{X}$ of $G$ as the input. Then the self-connected adjacency matrix is  $\bm{\hat{A}} = \bm{I} + \bm{A}$, the diagonal matrix is $\bm{\hat{D}} = \text{diag}(\bm{1}^\top \bm{\hat{A}})$, then the normalized self-connected adjacency matrix is $\bm{U} = \bm{\hat{D}}^{-\frac{1}{2}} \bm{\hat{A}} \bm{\hat{D}}^{-\frac{1}{2}}$. Let $\sigma(\cdot)$ be an activation function, 
then the hidden embedding $\bm{X}^{(m, L)}$ of the $l$-th layer is defined as follows 
\begin{equation}
    \begin{aligned}
    \bm{X}^{(m, l)} = \underbrace{\sigma(\bm{U} ... \sigma(\bm{U}}_{l~~\text{times}} \bm{X} \underbrace{\bm{W}^{(m, 1)}) ... \bm{W}^{(m, l)})}_{l~~\text{times}}, ~~ \forall~ l \in [L],
    \end{aligned}
\end{equation}
The pattern representation $\bm{z}^{(m)}$ of pattern $\mathcal{P}_m$ is obtained by
\begin{equation}
    \begin{aligned}
        \bm{z}^{(m)} = F(\bm{A}, \bm{X}; \mathcal{W}^{(m)}) = \text{avg-pool}(\bm{X}^{(m, L)}) = \frac{1}{n} \bm{1}^\top \bm{X}^{(m, L)} 
    \end{aligned}
\end{equation}

For a perturbed graph $\tilde{G}$, we use $\bm{\tilde{A}}$ and $\bm{\tilde{X}}$ to denote the adjacency matrix and feature matrix respectively. The corresponding self-connected adjacency matrix is  $\bm{\hat{A}'} = \bm{I} + \bm{\tilde{A}}$ and the degree matrix as $\bm{\hat{D}'} = \text{diag}(\bm{1}^\top \bm{\hat{A}'})$. Then the normalized self-connected adjacency matrix is $\bm{\tilde{U}} = \bm{\hat{D}'}^{-\frac{1}{2}} \bm{\hat{A}'} \bm{\hat{D}'}^{-\frac{1}{2}}$. The $l$-th layer hidden embedding of $\tilde{G}$ is defined as follows 
\begin{equation}
    \begin{aligned}
    \bm{\tilde{X}}^{(m, l)} = \underbrace{\sigma(\bm{\tilde{U}} ... \sigma(\bm{\tilde{U}} }_{l~~\text{times}} \bm{\tilde{X}} \underbrace{\bm{W}^{(m, 1)}) ... \bm{W}^{(m, l)})}_{l~~\text{times}}, ~~ \forall~ l \in [L],
    \end{aligned}
\end{equation}
The perturbed pattern representation $\bm{\tilde{z}}^{(m)}$ of pattern $\mathcal{P}_m$ is obtained by
\begin{equation}
    \begin{aligned}
        \bm{\tilde{z}}^{(m)} = F(\bm{\tilde{A}}, \bm{\tilde{X}}; \mathcal{W}^{(m)}) = \text{avg-pool}(\bm{\tilde{X}}^{(m, L)}) = \frac{1}{n} \bm{1}^\top \bm{\tilde{X}}^{(m, L)} 
    \end{aligned}
\end{equation}

\begin{lemma}
    Let $\bm{X}$ and $\bm{Y}$ be two square matrices, $\| \cdot \|_2$ be the spectral norm  and
    $\| \cdot \|_{F}$ be the Frobenius  norm , then $\|\bm{X}\|_2 \le \|\bm{X}\|_F$, 
    $\|\bm{X} \bm{Y}\|_2 \le \|\bm{X}\|_2 \|\bm{Y}\|_2$ and 
    $\|\bm{X} \bm{Y}\|_F \le \|\bm{X}\|_2 \|\bm{Y}\|_F$. 
\end{lemma}

\begin{lemma}[Inequalities] \label{lem: ineq-graph}
Some inequalities that will be used in our proof:
\begin{equation*}
    \begin{aligned}
        &\| \bm{U} \|_2  \leq (1 + \alpha)^{-1} ( 1 + \beta_A) \\
        &\| \bm{\tilde{U}} \|_2  \leq (1 + \alpha + \kappa)^{-1} (1 + \beta_A + \|\Delta_A \|_2) \\
        &\|\Delta_{U} \|_2 \leq 2(1 + \beta_A)(1 + \alpha)^{-1} \|\Delta_D \|_2 + (1 + \alpha + \kappa)^{-1} \| \Delta_A \|_2 \\
        & \|\Delta_{X^{(m, l)}} \|_F \leq \rho^{l} \beta_W^{l} \beta_X (1 + \alpha)^{-l} (1 + \beta_A + \|\Delta_A \|_2)^l \Bigl( 1 + 2l  \|\Delta_D \|_2 + l (1 + \beta_A + \|\Delta_A \|_2)^{-1}  \| \Delta_A \|_2) \Bigr)
    \end{aligned}
\end{equation*}
\end{lemma}

\begin{proof}
    Since the minimum node degree of $G$ is $\alpha$, then we have $\|\bm{\hat{D}}^{-\frac{1}{2}}\|_2 \leq (1 + \alpha)^{-\frac{1}{2}}$. Since $\|\bm{A}\|_2 \leq \beta_A$, then $\|\bm{\hat{A}}\|_2 \leq 1 + \beta_A$. We have
    \begin{equation}\label{eqn:ineq-U}
        \| \bm{U} \|_2 \leq \|\bm{\hat{D}}^{-\frac{1}{2}}\|_2 \|\bm{\hat{A}}\|_2 \|\bm{\hat{D}}^{-\frac{1}{2}}\|_2 \leq (1 + \alpha)^{-1} ( 1 + \beta_A).
    \end{equation}
    Similarly, since the effects of structural perturbation is $\kappa = \min(\bm{1}^\top \Delta_A)$, we have $\|\bm{\hat{D}'}^{-\frac{1}{2}}\|_2 \leq (1 + \alpha + \kappa)^{-\frac{1}{2}}$. Since $\|\bm{\tilde{A}'}\|_2 \leq \|\bm{\hat{A}}\|_2 + \|\Delta_A \|_2 \leq 1 + \beta_A + \|\Delta_A \|_2 $,  we obtain
    \begin{equation}\label{eqn:ineq-til-U}
        \| \bm{\tilde{U}} \|_2 \leq \|\bm{\hat{D}'}^{-\frac{1}{2}}\|_2 \|\bm{\hat{A}'}\|_2 \|\bm{\hat{D}'}^{-\frac{1}{2}}\|_2 \leq (1 + \alpha + \kappa)^{-1} (1 + \beta_A + \|\Delta_A \|_2).
    \end{equation}

Letting $\Delta_{U} = \bm{\tilde{U}} - \bm{U}$, we have
    \begin{equation}\label{eqn:ineq-det-U}
    \begin{aligned}
        \|\Delta_{U} \|_2  & = \|\bm{\tilde{U}} - \bm{U}\|_2\\ 
        & = \| \bm{\hat{D}'}^{-\frac{1}{2}} (\bm{\hat{A}} + \Delta_A) \bm{\hat{D}'}^{-\frac{1}{2}} -  \bm{\hat{D}}^{-\frac{1}{2}} \bm{\hat{A}} \bm{\hat{D}}^{-\frac{1}{2}}\|_2 \\
        & = \|\bm{\hat{D}'}^{-\frac{1}{2}} \bm{\hat{A}} \bm{\hat{D}'}^{-\frac{1}{2}} -
        \bm{\hat{D}'}^{-\frac{1}{2}} \bm{\hat{A}} \bm{\hat{D}}^{-\frac{1}{2}} + \bm{\hat{D}'}^{-\frac{1}{2}} \bm{\hat{A}} \bm{\hat{D}}^{-\frac{1}{2}} -\bm{\hat{D}}^{-\frac{1}{2}} \bm{\hat{A}} \bm{\hat{D}}^{-\frac{1}{2}} +
        \bm{\hat{D}'}^{-\frac{1}{2}} \Delta_A\bm{\hat{D}'}^{-\frac{1}{2}}\|_2 \\
        & \leq \|\bm{\hat{D}'}^{-\frac{1}{2}} \bm{\hat{A}} (\bm{\hat{D}'}^{-\frac{1}{2}} - \bm{\hat{D}}^{-\frac{1}{2}} )  \|_2 +  \|(\bm{\hat{D}'}^{-\frac{1}{2}} - \bm{\hat{D}}^{-\frac{1}{2}} ) \bm{\hat{A}} \bm{\hat{D}}^{-\frac{1}{2}}  \|_2
      +\| \bm{\hat{D}'}^{-\frac{1}{2}} \Delta_A\bm{\hat{D}'}^{-\frac{1}{2}}\|_2 \\
    & \leq (\|\bm{\hat{D}}^{-\frac{1}{2}}\|_2 + \|\bm{\hat{D}'}^{-\frac{1}{2}}\|_2) \|\bm{\hat{A}}\|_2 \|\bm{\hat{D}'}^{-\frac{1}{2}} - \bm{\hat{D}}^{-\frac{1}{2}} \|_2 + \| \bm{\hat{D}'}^{-\frac{1}{2}}\|_2 \| \Delta_A \|_2 \|\bm{\hat{D}'}^{-\frac{1}{2}}\|_2 \\
    & \leq ((1 + \alpha)^{-\frac{1}{2}} + (1 + \alpha + \kappa)^{-\frac{1}{2}}) (1 + \beta_A) 
    \|\bm{\hat{D}'}^{-\frac{1}{2}} - \bm{\hat{D}}^{-\frac{1}{2}} \|_2 + (1 + \alpha + \kappa)^{-1} \| \Delta_A \|_2 \\
    & \leq  2(1 + \beta_A)(1 + \alpha)^{-\frac{1}{2}} \|\bm{\hat{D}'}^{-\frac{1}{2}} - \bm{\hat{D}}^{-\frac{1}{2}} \|_2 + (1 + \alpha + \kappa)^{-1} \| \Delta_A \|_2\\
    & \leq 2(1 + \beta_A)(1 + \alpha)^{-\frac{1}{2}}(1 + \alpha + \kappa)^{-\frac{1}{2}}
    \|\bm{I} - \bm{\hat{D}'}^{\frac{1}{2}} \bm{\hat{D}}^{-\frac{1}{2}}\|_2 + (1 + \alpha + \kappa)^{-1} \| \Delta_A \|_2\\
    & = 2(1 + \beta_A)(1 + \alpha)^{-\frac{1}{2}}(1 + \alpha + \kappa)^{-\frac{1}{2}}
    \|\Delta_D \|_2 + (1 + \alpha + \kappa)^{-1} \| \Delta_A \|_2\\
    & \leq 2(1 + \beta_A)(1 + \alpha)^{-1} \|\Delta_D \|_2 + (1 + \alpha + \kappa)^{-1} \| \Delta_A \|_2
    \end{aligned}
\end{equation}
where $\Delta_D = \bm{I} - \bm{\hat{D}'}^{\frac{1}{2}} \bm{\hat{D}}^{-\frac{1}{2}} 
= \bm{I} - \text{diag}(\bm{1}^\top (\bm{I} + \bm{A} + \Delta_A))^{\frac{1}{2}} 
\text{diag}(\bm{1}^\top \bm{A})^{-\frac{1}{2}}$.

The $\bm{X}^{(m, l)}$ is the hidden embedding of the $l$-layer GCN of $F(\bm{A}, \bm(X); \mathcal{W}^{(m, l)})$, which is the representation learning function related to $\mathcal{P}_m$. Then we have
\begin{equation}\label{eqn:ineq-X}
    \begin{aligned}
        \|\bm{X}^{(m, l)} \|_F &= \|\sigma(\bm{U} \bm{X}^{(m, l-1)} \bm{W}^{(m, l)}) \|_F \\
        & \leq \rho \|\bm{U} \bm{X}^{(m, l-1)} \bm{W}^{(m, l)} \|_F \\
        & \leq \rho \|\bm{U} \|_2 \|\bm{X}^{(m, l-1)} \|_F \|\bm{W}^{(m, l)} \|_2\\
        & \leq \rho \beta_W (1 + \alpha)^{-1} ( 1 + \beta_A) \|\bm{X}^{(m, l-1)} \|_F \\
        & \leq \rho^l \beta_W^l  ( 1 + \beta_A)^l (1 + \alpha)^{-l} \|\bm{X} \|_F \\
        & \leq \rho^l \beta_W^l \beta_X ( 1 + \beta_A)^l (1 + \alpha)^{-l}
    \end{aligned}
\end{equation}

For $\Delta_{X^{(m, l)}} = \bm{\tilde{\bm{X}}}^{(m, l)} - \bm{X}^{(m, l)}$, we have
\begin{equation}\label{eqn:ineq-det-X}
    \begin{aligned}
        \|\Delta_{X^{(m, l)}} \|_F &=  \|\bm{\tilde{\bm{X}}}^{(m, l)} - \bm{X}^{(m, l)} \|_F \\
        & = \|\sigma(\tilde{\bm{U}} \bm{\tilde{\bm{X}}}^{(m, l-1)} \bm{W}^{(l)}) -
                \sigma(\bm{U} \bm{X}^{(l-1)} \bm{W}^{(l)}) \|_F \\
        &\leq \rho \|\tilde{\bm{U}} \tilde{\bm{X}}^{(m,l-1)} - \bm{U} \bm{X}^{(m,l-1)} \|_F \|\bm{W}^{(m,l)} \|_2 \\
        &\leq \rho \beta_W  \left( \|\tilde{\bm{U}} \|_2 \|\Delta_{X^{(m,l-1)}} \|_F  
        + \|\Delta_{U} \|_2 \|\bm{X}^{(m,l-1)} \|_F \right) \\
        &\leq \rho^2 \beta_W^2 \|\tilde{\bm{U}} \|_2^2  \|\Delta_{X^{(m,l-2)}} \|_F + \rho^2 \beta_W^2 \|\tilde{\bm{U}} \|_2 \|\Delta_{U} \|_2 \|\bm{X}^{(m,l-2)} \|_F +\rho \beta_W \|\Delta_{U} \|_2 \|\bm{X}^{(m,l-1)} \|_F \\
        &\leq \rho^{l} \beta_W^{l}  \|\tilde{\bm{U}} \|_2^{l}   \|\Delta_{X} \|_F + \sum_{k = 1}^{l} \rho^k \beta_W^k \|\tilde{\bm{U}} \|_2^{k-1} \|\Delta_{U} \|_2 \|\bm{X}^{(m, l-k)} \|_F \\
    &\leq  \rho^l \beta_W^l     (1 + \beta_A + \|\Delta_A \|_2)^{l-1} (1 + \alpha)^{-l} \left[ (1 + \beta_A + 2 \|\Delta_A \|_2)\|\Delta_{X} \|_F + 2l\beta_X(1 + \beta_A) \|\Delta_D \|_2  \right]
    \end{aligned}
\end{equation}
\end{proof}

\subsection{Learning Graph Representations via Sampling Subgraphs}
In this section, we consider learning the graph representation $\bm{g}$ and $\bm{\tilde{g}}$ respectively by sampling subgraphs of graph patterns. That is, we analyse $F(\bm{A}_{S}, \bm{X}_S; \mathcal{W}^{(m)})$ and $F(\bm{\tilde{A}}_{\tilde{S}}, \bm{\tilde{X}_{\tilde{S}}}; \mathcal{W}^{(m)})$. And then we provide the upper bound of $\|\bm{\tilde{g}} - \bm{g} \|$. 

Let $S$ be a subgraph of graph $G$ and $\tilde{S}$ be a subgraph of graph $\tilde{G}$. 
Let $\Delta_{A_S}$ and $\Delta_{X_S}$ be some perturbations on adjacency matrix and node attributes, then the perturbed graph is denoted as $\tilde{S} = (\bm{A}_S + \Delta_{A_S}, \bm{X}_S + \Delta_{X_S})$. 

\paragraph{Assumptions and Notations:} Let $\bm{\tilde{A}} = \bm{A} + \Delta_A$ and $\bm{\tilde{X}} = \bm{X} + \Delta_X$. We suppose that $\|\bm{A}\|_2 \leq \beta_A$, $\|\bm{X}\|_F \leq \beta_B$ and $\|\bm{W}^{(m, l)}\|_2 \leq \beta_W,~(\forall~m \in [M],~l \in [L])$, the activation $\sigma(\cdot)$ of GCN is $\rho$-Lipschitz continuous. We denote the minimum node degree of $G$ as $\alpha$,  the effects of structural perturbation as $\kappa = \min(\bm{1}^\top \Delta_A)$, and $\Delta_D := \bm{I} - \text{diag}(\bm{1}^\top (\bm{I} + \bm{A} + \Delta_A))^{\frac{1}{2}} \text{diag}(\bm{1}^\top \bm{A})^{-\frac{1}{2}}$.
We present the following useful lemmas.
\begin{lemma}[Eigenvalue Interlacing Theorem \cite{hwang2004cauchy}]
    Suppose \( A \in \mathbb{R}^{n \times n} \) is symmetric. Let \( B \in \mathbb{R}^{m \times m} \) with \( m < n \) be a principal submatrix (obtained by deleting both the \( i \)-th row and \( i \)-th column for some value of \( i \)). Suppose \( A \) has eigenvalues \( \lambda_1 \leq \cdots \leq \lambda_n \) and \( B \) has eigenvalues \( \beta_1 \leq \cdots \leq \beta_m \). Then
$$\lambda_k \leq \beta_k \leq \lambda_{k+n-m} \quad \text{for } k = 1, \cdots, m.$$
\end{lemma}
\begin{lemma}\label{lem:subgraph-assump}
    Since $\bm{X}_S$ and $\Delta_{X_S}$ are submatrices of $\bm{X}$ and $\Delta_{X}$ respectively, then we have
    $$\|\bm{X}_S \|_F \leq \| \bm{X} \|_F, ~~\text{and}~~ \|\Delta_{X_S} \|_F \leq \|\Delta_{X} \|_F.$$
    Let $\Delta_{D_S} := \bm{I} - \text{diag}(\bm{1}^\top (\bm{I} + \bm{A}_S + \Delta_{A_S}))^{\frac{1}{2}} \text{diag}(\bm{1}^\top \bm{A}_S)^{-\frac{1}{2}}$.
    Base on the Eigenvalue Interlacing Theorem, for any subgraph $S$ of graph $G$, since $\bm{A}_S$, $\Delta_{A_S}$, $\Delta_{D_S}$ are principal submatrices of $\bm{A}$, $\Delta_{A}$, $\Delta_{D}$ respectively, then we have 
$$\|\bm{A}_S\|_2 \leq \|\bm{A}\|_2 \leq \beta_A, ~~~~  
\|\Delta_{A_S}\|_2 \leq \|\Delta_{A}\|_2, ~~~~ 
\|\Delta_{D_S}\|_2 \leq \|\Delta_{D}\|_2.$$
\end{lemma}

\paragraph{Notations:}
For a subgraph $S$ of graph $G$, the self-connected adjacency matrix is  $\bm{\hat{A}}_S = \bm{I} + \bm{A}_S$, the degree matrix is $\bm{\hat{D}}_S = \text{diag}(\bm{1}^\top \bm{\hat{A}}_S)$, and the normalized self-connected adjacency matrix is $\bm{U}_S = \bm{\hat{D}}_S^{-\frac{1}{2}} \bm{\hat{A}}_S \bm{\hat{D}}_S^{-\frac{1}{2}}$. 

For a subgraph $\tilde{S}$ of graph $\tilde{G}$, we define some notations here. 
We denote the self-connected adjacency matrix as  $\bm{\hat{A}}'_{\tilde{S}} = \bm{I} + \bm{\tilde{A}}_{\tilde{S}}$, the diagonal matrix as $\bm{\hat{D}'}_{\tilde{S}} = \text{diag}(\bm{1}^\top \bm{\hat{A}'}_{\tilde{S}})$, and the normalized self-connected adjacency matrix as $\bm{\tilde{U}}_{\tilde{S}} = \bm{\hat{D}'}_{\tilde{S}} {^{-\frac{1}{2}}} \bm{\hat{A}'}_{\tilde{S}} \bm{\hat{D}'}_{\tilde{S}} {^{-\frac{1}{2}}}$.
We also denote $\Delta_{U_S} = \bm{\tilde{U}}_{\tilde{S}}-  \bm{U}_S$ and $\Delta_{X_S^{(m, l)}} = \bm{\tilde{\bm{X}}}_{\tilde{S}}{^{(m, l)}} - \bm{X}_S^{(m, l)}$.

\begin{lemma}[Inequalities] \label{lem: ineq-subgraph}  
Base on Lemma \ref{lem:subgraph-assump}, for any subgraph $S$ of graph $G$, the inequalities in the Lemma \ref{lem: ineq-graph} still holds for $S$, shown as follows:
\begin{equation}
    \begin{aligned}
        &\| \bm{U}_S \|_2  \leq (1 + \alpha)^{-1} ( 1 + \beta_A) \\
        &\| \bm{\tilde{U}}_S \|_2  \leq (1 + \alpha + \kappa)^{-1} (1 + \beta_A + \|\Delta_A \|_2) \\
        &\|\Delta_{U_S} \|_2 \leq 2(1 + \beta_A)(1 + \alpha)^{-1} \|\Delta_D \|_2 + (1 + \alpha + \kappa)^{-1} \| \Delta_A \|_2 \\
        &\|\bm{X}_S^{(m, l)} \|_F  \leq \rho^l \beta_W^l \beta_X ( 1 + \beta_A)^l (1 + \alpha)^{-l}\\
        &\|\Delta_{X_S^{(m, l)}} \|_F \leq   \rho^l \beta_W^l     (1 + \beta_A + \|\Delta_A \|_2)^{l-1} (1 + \alpha)^{-l} \\
        &\left[ (1 + \beta_A + 2 \|\Delta_A \|_2)\|\Delta_{X} \|_F + 2l\beta_X(1 + \beta_A) \|\Delta_D \|_2  \right]
    \end{aligned}
\end{equation}
\end{lemma}
\begin{proof}
    The proof is mainly based on Lemma \ref{lem:subgraph-assump}. 
    Similar to Inequality \eqref{eqn:ineq-U}, we have
    \begin{equation}
    \begin{aligned}
                \| \bm{U}_S \|_2 &\leq \|\bm{\hat{D}}_S^{-\frac{1}{2}}\|_2 \|\bm{\hat{A}}_S\|_2 \|\bm{\hat{D}}_S^{-\frac{1}{2}}\|_2 \\
        &\leq \|\bm{\hat{D}}^{-\frac{1}{2}}\|_2 \|\bm{\hat{A}}\|_2 \|\bm{\hat{D}}^{-\frac{1}{2}}\|_2 \\
        &\leq (1 + \alpha)^{-1} ( 1 + \beta_A).
    \end{aligned}
    \end{equation}

    Similar to Inequality \eqref{eqn:ineq-til-U}, we have
    \begin{equation}
    \begin{aligned}
        \| \bm{\tilde{U}}_S \|_2 &\leq \|\bm{\hat{D}'}_{\tilde{S}}{^{-\frac{1}{2}}}\|_2 \|\bm{\hat{A}'}\|_2 \|\bm{\hat{D}'}_{\tilde{S}}{^{-\frac{1}{2}}}\|_2 \\
        &\leq \|\bm{\hat{D}'}^{-\frac{1}{2}}\|_2 \|\bm{\hat{A}'}\|_2 \|\bm{\hat{D}'}^{-\frac{1}{2}}\|_2\\
        &\leq (1 + \alpha + \kappa)^{-1} (1 + \beta_A + \|\Delta_A \|_2).
    \end{aligned}
    \end{equation}

    Similar to Inequality \eqref{eqn:ineq-det-U}, we have
    \begin{equation}
    \begin{aligned}
    \|\Delta_{U} \|_2 & \leq (\|\bm{\hat{D}}_S^{-\frac{1}{2}}\|_2 + \|\bm{\hat{D}'}_{\tilde{S}}{^{-\frac{1}{2}}}\|_2) \|\bm{\hat{A}}\|_2 \|\bm{\hat{D}'}_{\tilde{S}}{^{-\frac{1}{2}}} - \bm{\hat{D}}_S^{-\frac{1}{2}} \|_2  + \| \bm{\hat{D}'}_{\tilde{S}}{^{-\frac{1}{2}}} \| \Delta_A \|_2 \|\bm{\hat{D}'}_{\tilde{S}}{^{-\frac{1}{2}}}\|_2 \\
    & \leq (\|\bm{\hat{D}}^{-\frac{1}{2}}\|_2 + \|\bm{\hat{D}'}^{-\frac{1}{2}}\|_2) \|\bm{\hat{A}}\|_2 \|\bm{\hat{D}'}^{-\frac{1}{2}} - \bm{\hat{D}}^{-\frac{1}{2}} \|_2 + \| \bm{\hat{D}'}^{-\frac{1}{2}}\|_2 \| \Delta_A \|_2 \|\bm{\hat{D}'}^{-\frac{1}{2}}\|_2 \\
    & \leq 2(1 + \beta_A)(1 + \alpha)^{-1} \|\Delta_D \|_2 + (1 + \alpha + \kappa)^{-1} \| \Delta_A \|_2
    \end{aligned}
\end{equation}

Similar to Inequality \eqref{eqn:ineq-X}, we have
\begin{equation}
    \begin{aligned}
\|\bm{X}_S^{(m, l)} \|_F 
        & \leq \rho \|\bm{U}_S \|_2 \|\bm{X}_S^{(m, l-1)} \|_F \|\bm{W}^{(m, l)} \|_2\\
        & \leq \rho \|\bm{U} \|_2 \|\bm{X}^{(m, l-1)} \|_F \|\bm{W}^{(m, l)} \|_2\\
        & \leq \rho^l \beta_W^l \beta_X ( 1 + \beta_A)^l (1 + \alpha)^{-l}
    \end{aligned}
\end{equation}

Similar to Inequality \eqref{eqn:ineq-det-X}, we have
\begin{equation}
    \begin{aligned}
        \|\Delta_{X_S^{(m, l)}} \|_F 
        &\leq \rho^2 \beta_W^2 \|\tilde{\bm{U}}_S \|_2^2  \|\Delta_{X_S^{(l-2)}} \|_F + \rho^2 \beta_W^2 \|\tilde{\bm{U}}_S \|_2 \|\Delta_{U_S} \|_2 \|\bm{X}_S^{(l-2)} \|_F + \rho \beta_W \|\Delta_{U_S} \|_2 \|\bm{X}_S^{(l-1)} \|_F \\
        &\leq \rho^2 \beta_W^2 \|\tilde{\bm{U}} \|_2^2  \|\Delta_{X^{(l-2)}} \|_F + \rho^2 \beta_W^2 \|\tilde{\bm{U}} \|_2 \|\Delta_{U} \|_2 \|\bm{X}^{(l-2)} \|_F +\rho \beta_W \|\Delta_{U} \|_2 \|\bm{X}^{(l-1)} \|_F \\
     & \leq  \rho^l \beta_W^l     (1 + \beta_A + \|\Delta_A \|_2)^{l-1} (1 + \alpha)^{-l} \left[ (1 + \beta_A + 2 \|\Delta_A \|_2)\|\Delta_{X} \|_F + 2l\beta_X(1 + \beta_A) \|\Delta_D \|_2  \right]
    \end{aligned}
\end{equation}
\end{proof}

Finally, we can prove our theorem of robustness analysis in the main paper using Lemma \ref{lem: ineq-subgraph} as follows.

\begin{proof}
Given a pattern sampling set $\mathcal{S}^{(m)}$, we assume the $S^*$ satisfies
$$S^* = \argmax_{S \in \mathcal{S}^{(m)}} \|\Delta_{X_{S}^{(m, L)}} \|_F.$$
Since the Lemma \ref{lem: ineq-subgraph} holds for any subgraph $S$, we have 
\begin{equation*}
    \begin{aligned}
        & \|\Delta_{X_{S^*}^{(m, l)}} \|_F \leq  \rho^l \beta_W^l     (1 + \beta_A + \|\Delta_A \|_2)^{l-1} (1 + \alpha)^{-l} \left[ (1 + \beta_A + 2 \|\Delta_A \|_2)\|\Delta_{X} \|_F + 2l\beta_X(1 + \beta_A) \|\Delta_D \|_2  \right]
    \end{aligned}
\end{equation*}

Then the upper bound of $\|\bm{\tilde{g}}- \bm{g}  \| $ is given by
    \begin{equation}
    \begin{aligned}
        \|\bm{\tilde{g}}- \bm{g}  \| 
        &=  \left\| \sum_{m = 1}^M \lambda_m~ (\bm{\tilde{z}}^{(m)}- \bm{z}^{(m)}) \right\| \\
        & \leq  \sum_{m = 1}^M \lambda_m~ \|\bm{\tilde{z}}^{(m)}- \bm{z}^{(m)} \| \\
        & = \frac{1}{Q} \sum_{m = 1}^M \lambda_m~\bigg\|  \sum_{S \in \mathcal{S}^{(m)}} F(\bm{\tilde{A}}_S, \bm{\tilde{X}}_S; \mathcal{W}^{(m)}) 
        - \sum_{S \in \mathcal{S}^{(m)}} F(\bm{A}_S, \bm{X}_S; \mathcal{W}^{(m)}) \bigg\| \\
        & \leq \frac{1}{Q} \sum_{m = 1}^M \lambda_m~ \sum_{S \in \mathcal{S}^{(m)}} 
        \bigg\| F(\bm{\tilde{A}}_S, \bm{\tilde{X}}_S; \mathcal{W}^{(m)}) - F(\bm{A}_S, \bm{X}_S; \mathcal{W}^{(m)})\bigg\| \\
        & = \frac{1}{Q} \sum_{m = 1}^M \lambda_m~ \sum_{S \in \mathcal{S}^{(m)}}
        \frac{1}{n} \left\|\bm{1}^\top (\bm{\tilde{X}}_S^{(m, L)} - \bm{X}_S^{(m, L)}) \right\|_F \\
        & \leq  \frac{1}{Q} \sum_{m = 1}^M \lambda_m~ \frac{1}{n} \sum_{S \in \mathcal{S}^{(m)}} \|\bm{1} \| \left\|\bm{\tilde{X}}_S^{(m, L)} - \bm{X}_S^{(m, L)} \right\|_F \\
        &= \frac{1}{Q \sqrt{n}} \sum_{m = 1}^M \lambda_m~ \sum_{S \in \mathcal{S}^{(m)}} \left\|\Delta_{X_S^{(m, L)}} \right\|_F \\
        & \leq \frac{1}{Q \sqrt{n}} \sum_{m = 1}^M \lambda_m~ Q \left\|\Delta_{X_{S^*}^{(m, L)}} \right\|_F \\
        & \leq \frac{1}{\sqrt{n}}  \rho^l \beta_W^l     (1 + \beta_A + \|\Delta_A \|_2)^{l-1}(1 + \alpha)^{-l} \bigg[ (1 + \beta_A + 2 \|\Delta_A \|_2)\|\Delta_{X} \|_F + 2L\beta_X(1 + \beta_A) \|\Delta_D \|_2  \bigg]
    \end{aligned}
\end{equation}
\end{proof}

\section{Proof for Generalization Analysis of Supervised Loss}\label{app:generalization}
Before providing our theorem, we need to provide the classification loss function $f_c$.
\paragraph{Classification loss function $f_c$:} We use a linear classifier with parameter $\bm{W}_C \in \mathbb{R}^{d \times C}$ and use softmax as the activation function as the classification function $f_c$, i.e., $\bm{\hat{y}} = \text{softmax}(\bm{g} \bm{W}_C)$. We suppose that $\|\bm{W}_C\|_2 \leq \beta_C $.

Then the classification loss is as follows
\begin{equation}
\begin{aligned}
        \ell_{\text{CE}}(\bm{\lambda}, \mathbb{W}) &= \text{cross-entropy}(\bm{y}, \bm{\hat{y}}) \\
    &= \text{cross-entropy}(\bm{y}, \text{softmax}(\bm{g} \bm{W}_C)).
\end{aligned}
\end{equation}
To simplify the proof, we rewrite supervised loss $\ell_{\text{CE}}(\bm{\lambda}, \mathbb{W})$ function as 
\begin{equation}
    \begin{aligned}
        \varphi(\bm{g} \bm{W}_C) &:= \text{cross-entropy}(\bm{y}, \bm{\hat{y}}) \\
    &= \text{cross-entropy}(\bm{y}, \text{softmax}(\bm{g} \bm{W}_C)).
    \end{aligned}
\end{equation}
\begin{lemma}\label{lem:ce-loss-lips}
    Let $\bm{v}$ be a vector, there exits a positive constant $\tau$ such that $\varphi(\bm{v})$ is a $\tau$-Lipschitz continuous function. 
\end{lemma}

\paragraph{Generalization Error}
Let $\mathcal{D} := \{G_1, ..., G_{|\mathcal{D}|} \}$ be the training data. By removing the $i$-th graph of $\mathcal{D}$, we have $\mathcal{D}^{\backslash i} = \{G_1, ..., G_{i-1}, G_{i+1}, ..., G_{|\mathcal{D}| - 1} \}$. Let $\bm{\lambda}_{\mathcal{D}}$ and $\mathcal{W}_{\mathcal{D}} := \{\bm{W}_C, \bm{W}_{\mathcal{D}}^{(m, l)}, ~\forall~m \in [M],~l\in[L]\}$ be the parameters trained on $\mathcal{D}$. Let $\bm{\lambda}_{\mathcal{D}^{\backslash i}}$ and $\mathcal{W}_{\mathcal{D}^{\backslash i}} := \{\bm{W}_{C^{\backslash i}}, \bm{W}_{\mathcal{D}^{\backslash i}}^{(m, l)}, ~\forall~m \in [M],~l\in[L]\}$ be the parameters trained on $\mathcal{D}^{\backslash i}$. Then our goal is to find a $\eta$ such that 
\begin{equation}
     |\ell_{\text{CE}}(\bm{\lambda}_{\mathcal{D}}, \mathcal{W}_{\mathcal{D}}; G) - \ell_{\text{CE}}(\bm{\lambda}_{\mathcal{D}^{\backslash i}}, \mathcal{W}_{\mathcal{D}^{\backslash i}}; G)| \leq \eta
\end{equation}

\begin{theorem}
    Given a graph $G$, let $\bm{g}$ be the graph representations learned with parameter $\bm{\lambda}_{\mathcal{D}}$ and $\mathcal{W}_{\mathcal{D}}$ and $\bm{g}^{\backslash i}$ be the graph representations learned with parameter $\bm{\lambda}_{\mathcal{D}^{\backslash i}}$ and $\mathcal{W}_{\mathcal{D}^{\backslash i}}$. 

    To simplify the proof, we denote that $\hat{\beta}_W = \max(\hat{\beta}_{W\mathcal{D}}, \hat{\beta}_{W\mathcal{D}^{\backslash i}})$, where
\begin{equation*}
    \begin{aligned}
        & \hat{\beta}_{W\mathcal{D}} = \max_{m \in [M], l \in [L]} \|W^{(m, l)}_{\mathcal{D}}\|_2, ~~\text{and}~~\\
        &\hat{\beta}_{W\mathcal{D}^{\backslash i}} = \max_{m \in [M], l \in [L]} \|W^{(m, l)}_{\mathcal{D}^{\backslash i}}\|_2.
    \end{aligned}
\end{equation*}
We also denote that
$$\hat{\beta}_{\Delta W} = \max_{m \in [M], l \in [L]} \|\mathcal{W}^{(m, l)}_{\mathcal{D}} -  \mathcal{W}^{(m, l)}_{\mathcal{D}^{\backslash i}}\|_2.$$
Then we have
\begin{equation*}
    \begin{aligned}
        \eta &=  \frac{\tau}{\sqrt{n}} \rho^L \hat{\beta}_W^{L-1} \beta_X ( 1 + \beta_A)^L (1 + \alpha)^{-L} \big[ \hat{\beta}_W \| \bm{W}_C - \bm{W}_{C^{\backslash i}} \|_2 +   \|\bm{W}_{C^{\backslash i}} \|_2  \left(\hat{\beta}_W  \|\bm{\lambda}_{\mathcal{D}} - \bm{\lambda}_{\mathcal{D}^{\backslash i}} \| + L \hat{\beta}_{\Delta W} \|\bm{\lambda}_{\mathcal{D}^{\backslash i}} \|  \right) \big]
    \end{aligned}
\end{equation*}
\end{theorem}

\begin{proof}
We provide two lemmas used in our proof
    \begin{lemma}
        $\| \bm{g} \|\leq \frac{1}{\sqrt{n}} \rho^L \hat{\beta}_W^L \beta_X ( 1 + \beta_A)^L (1 + \alpha)^{-L}$
    \end{lemma}

    \begin{lemma}
    \begin{equation*}
        \begin{aligned}
            \|\bm{g} - \bm{g}^{\backslash i} \| & \leq \frac{1}{\sqrt{n}} \rho^L \hat{\beta}_W^{L-1}  \beta_X ( 1 + \beta_A)^L (1 + \alpha)^{-L} \left(\hat{\beta}_W  \|\bm{\lambda}_{\mathcal{D}} - \bm{\lambda}_{\mathcal{D}^{\backslash i}} \| + L \hat{\beta}_{\Delta W} \|\bm{\lambda}_{\mathcal{D}^{\backslash i}} \|  \right)
        \end{aligned}
    \end{equation*}
    \end{lemma}

\paragraph{The main proof of our Theorem}
    \begin{equation}
        \begin{aligned}
    &|\ell_{\text{CE}}(\bm{\lambda}_{\mathcal{D}}, \mathcal{W}_{\mathcal{D}}; G) - \ell_{\text{CE}}(\bm{\lambda}_{\mathcal{D}^{\backslash i}}, \mathcal{W}_{\mathcal{D}^{\backslash i}}; G)| \\
    =& \|\varphi(\bm{g}^{\backslash i} \bm{W}_{C^{\backslash i}}) - \varphi(\bm{g} \bm{W}_C)\| \\
    \leq& \tau \| \bm{g} \bm{W}_C - \bm{g}^{\backslash i} \bm{W}_{C^{\backslash i}} \| \\
    =& \tau \| \bm{g} \bm{W}_C - \bm{g} \bm{W}_{C^{\backslash i}} + \bm{g} \bm{W}_{C^{\backslash i}} - \bm{g}^{\backslash i} \bm{W}_{C^{\backslash i}} \| \\
     \leq& \tau \| \bm{g} \| \| \bm{W}_C - \bm{W}_{C^{\backslash i}} \|_2 + \tau \|\bm{g} - \bm{g}^{\backslash i} \| \|\bm{W}_{C^{\backslash i}} \|_2 \\
    \leq & \tau \| \bm{W}_C - \bm{W}_{C^{\backslash i}} \|_2 \frac{1}{\sqrt{n}} \rho^L \hat{\beta}_W^L \beta_X ( 1 + \beta_A)^L (1 + \alpha)^{-L}  \\
    &~~+ \tau \|\bm{W}_{C^{\backslash i}} \|_2 \frac{1}{\sqrt{n}} \rho^L \hat{\beta}_W^{L-1}  \beta_X ( 1 + \beta_A)^L (1 + \alpha)^{-L} \left(\hat{\beta}_W  \|\bm{\lambda}_{\mathcal{D}} - \bm{\lambda}_{\mathcal{D}^{\backslash i}} \| + L \hat{\beta}_{\Delta W} \|\bm{\lambda}_{\mathcal{D}^{\backslash i}} \|  \right) \\
     = & \frac{\tau}{\sqrt{n}} \rho^L \hat{\beta}_W^{L-1} \beta_X ( 1 + \beta_A)^L (1 + \alpha)^{-L}  \bigg[ \hat{\beta}_W \| \bm{W}_C - \bm{W}_{C^{\backslash i}} \|_2 +   \|\bm{W}_{C^{\backslash i}} \|_2 \left(\hat{\beta}_W  \|\bm{\lambda}_{\mathcal{D}} - \bm{\lambda}_{\mathcal{D}^{\backslash i}} \| + L \hat{\beta}_{\Delta W} \|\bm{\lambda}_{\mathcal{D}^{\backslash i}} \|  \right) \bigg]
    \end{aligned}
    \end{equation}
    Since $\sum_{i=1}^M\lambda_i\leq 1$ and $\lambda_i\geq 0$, we have $\|\boldsymbol{\lambda}\|\leq 1$ and $\|\boldsymbol{\lambda}-\boldsymbol{\lambda}_{\mathcal{D}^{\backslash i}}\|\leq 2$. This finished the proof.
\end{proof}

\subsection{Proof for Lemmas}
\begin{lemma}
    Let $\bm{v}$ be a vector, there exits a positive constant $\tau$ such that $\varphi(\bm{v})$ is a $\tau$-Lipschitz continuous function. 
\end{lemma}
\begin{proof}
\textbf{Step 1: Softmax is Lipschitz}
The softmax function is known to be Lipschitz continuous. Specifically, there exists a constant \( K \) such that:
\[
\|\text{softmax}(v) - \text{softmax}(w)\|_1 \leq L_1 \|v - w\|_2,
\]
where \( \|\cdot\|_1 \) is the \( \ell_1 \)-norm and \( \|\cdot\|_2 \) is the \( \ell_2 \)-norm.
For the \( \ell_1 \)-norm, \( L_1 \) can be bounded by 1, but generally, for different norms, the exact Lipschitz constant might vary.

\textbf{Step 2: Cross-Entropy is Lipschitz on the Simplex}
Given \( \mathbf{q} = \text{softmax}(v) \) and \( \mathbf{r} = \text{softmax}(w) \), we need to check the Lipschitz continuity of the cross-entropy loss function with respect to these distributions:
\[
\left| \text{cross-entropy}(\mathbf{p}, \mathbf{q}) - \text{cross-entropy}(\mathbf{p}, \mathbf{r}) \right| \leq L_2 \|\mathbf{q} - \mathbf{r}\|.
\]
The cross-entropy loss is a convex function and it is smooth with respect to the probability distributions \( \mathbf{q} \) and \( \mathbf{r} \). Given the boundedness of the probability values (since \( \mathbf{q} \) and \( \mathbf{r} \) lie in the probability simplex), the gradient of the cross-entropy loss is also bounded.

\textbf{Combining Steps}
Since both the softmax function and the cross-entropy loss function are Lipschitz continuous, their composition will also be Lipschitz continuous. Therefore, there exists a constant \( \tau = L_1 L_2 \) such that:
\[
|\varphi(v) - \varphi(w)| \leq \tau \|v - w\|.
\]

Hence, \( \varphi(v) = \text{cross-entropy}(\text{softmax}(v)) \) is $\tau$-Lipschitz continuous.
\end{proof}

\begin{lemma}
    $\| \bm{g} \| \leq \frac{1}{\sqrt{n}} \rho^L \hat{\beta}_W^L \beta_X ( 1 + \beta_A)^L (1 + \alpha)^{-L}$
\end{lemma}

\begin{proof}
Given a pattern sampling set $\mathcal{S}^{(m)}$, we assume the $S^*$ satisfies
$$S^* = \argmax_{S \in \mathcal{S}^{(m)}} \|\bm{X}_{S}^{(L)} \|_F.$$
Since the Lemma \ref{lem: ineq-subgraph} holds for any subgraph $S$, then we have 
$$\|\bm{X}_{S^*}^{(m, l)} \|_F \leq \rho^l \hat{\beta}_W^l \beta_X ( 1 + \beta_A)^l (1 + \alpha)^{-l}.$$
Then, we have
    \begin{equation} \label{eqn:general-g}
        \begin{aligned}
         \| \bm{g} \|   &=  \| \sum_{m = 1}^M \lambda_m~  \bm{z}^{(m)} \| 
         \leq  \sum_{m = 1}^M \lambda_m~ \| \bm{z}^{(m)} \| \\
        &= \frac{1}{Q} \sum_{m = 1}^M \lambda_m~\|  \sum_{S \in \mathcal{S}^{(m)}} F(\bm{A}_S, \bm{X}_S; \mathcal{W}^{(m)}) \| \\
        & \leq \frac{1}{Q} \sum_{m = 1}^M \lambda_m~ \sum_{S \in \mathcal{S}^{(m)}} 
        \|  F(\bm{A}_S, \bm{X}_S; \mathcal{W}^{(m)}) \| \\
        &= \frac{1}{Q} \sum_{m = 1}^M \lambda_m~ \sum_{S \in \mathcal{S}^{(m)}}
        \frac{1}{n} \|\bm{1}^\top ( \bm{X}_S^{(m,L)}) \|_F \\
        & \leq  \frac{1}{Q} \sum_{m = 1}^M \lambda_m~ \frac{1}{n} \sum_{S \in \mathcal{S}^{(m)}} \|\bm{1} \|_2 \|\bm{X}_S^{(m,L)} \|_F \\
        & = \frac{1}{Q \sqrt{n}} \sum_{m = 1}^M \lambda_m~ \sum_{S \in \mathcal{S}^{(m)}} \|\bm{X}_S^{(m, L)} \|_F  \\
        & \leq \frac{1}{\sqrt{n}} \sum_{m = 1}^M \lambda_m~  \|\bm{X}_{S^*}^{(m, L)} \|_F \\
        & \leq \frac{1}{\sqrt{n}} \rho^L \hat{\beta}_W^L \beta_X ( 1 + \beta_A)^L (1 + \alpha)^{-L}
        \end{aligned}
    \end{equation}
\end{proof}

\begin{lemma}
\begin{equation*}
    \begin{aligned}
        \|\bm{g} - \bm{g}^{\backslash i} \| \leq& \frac{1}{\sqrt{n}} \rho^L \hat{\beta}_W^{L-1}  \beta_X ( 1 + \beta_A)^L (1 + \alpha)^{-L} \\
        &\left(\hat{\beta}_W  \|\bm{\lambda}_{\mathcal{D}} - \bm{\lambda}_{\mathcal{D}^{\backslash i}} \| + L \hat{\beta}_{\Delta W} \|\bm{\lambda}_{\mathcal{D}^{\backslash i}} \|  \right)
    \end{aligned}
\end{equation*}
\end{lemma}

\begin{proof}
To simplify the proof, we denote
\begin{equation}
    \begin{aligned}
        \hat{\beta}_W &= \max\{\max_{m \in [M], l \in [L]} \|\bm{W}^{(m, l)}_{\mathcal{D}}\|_2, \max_{m \in [M], l \in [L]} \|\bm{W}^{(m, l)}_{\mathcal{D}^{\backslash i}}\|_2\} \\
\hat{\beta}_{\Delta W} &= \max_{m \in [M], l \in [L]} \|\mathcal{W}^{(m, l)}_{\mathcal{D}} -  \mathcal{W}^{(m, l)}_{\mathcal{D}^{\backslash i}}\|_2.
    \end{aligned}
\end{equation}

Let $\bm{X}_{S \mathcal{D}}^{(m,l)}$ be the embedding features of the $l$-th layer GCN with the parameter $\mathcal{W}^{(m)}_{\mathcal{D}}$ learned from dataset $\mathcal{D}$. 
Let $\bm{X}_{S \mathcal{D}^{\backslash i}}^{(m,l)}$ be the embedding features of the $l$-th layer GCN with the parameter $\mathcal{W}^{(m)}_{\mathcal{D}^{\backslash i}}$ learned from dataset $\mathcal{D}^{\backslash i}$. 

We denote $\bm{Z}_{\mathcal{D}} = [\bm{z}_{\mathcal{D}}^{(1)}, ..., \bm{z}_{\mathcal{D}}^{(m)}]^\top$ and 
$\bm{Z}_{\mathcal{D}^{\backslash i}} = [\bm{z}_{\mathcal{D}^{\backslash i}}^{(1)}, ..., \bm{z}_{\mathcal{D}^{\backslash i}}^{(m)}]^\top $. Let
$$q_1 = \argmax_{m \in [M]} \|\bm{z}_{\mathcal{D}}^{(m)}\|, ~~~ 
q_2 = \argmax_{m \in [M]} \| \bm{z}_{\mathcal{D}}^{(l_2)} - \bm{z}_{\mathcal{D}^{\backslash i}}^{(l_2)}\|. $$
Then we have
$$\|\bm{Z}_{\mathcal{D}} \|_2 \leq  \| \bm{z}_{\mathcal{D}}^{(q_1)}\|, ~~~ 
\|\bm{Z}_{\mathcal{D}} - \bm{Z}_{\mathcal{D}^{\backslash i}} \|_2 \leq \| \bm{z}_{\mathcal{D}}^{(q_2)} - \bm{z}_{\mathcal{D}^{\backslash i}}^{(q_2)}\|.$$

Similar to inequality \eqref{eqn:general-g}, we have
\begin{equation}
    \begin{aligned}
        \| \bm{z}_{\mathcal{D}}^{(q_1)}\| \leq \frac{1}{\sqrt{n}} \rho^L \hat{\beta}_W^L \beta_X ( 1 + \beta_A)^L (1 + \alpha)^{-L}
    \end{aligned}
\end{equation}

Denote $\Delta_{X^{(q_2, l)}_{S\mathcal{D}}} := \bm{X}_{S \mathcal{D}}^{(q_2,l)} -  \bm{X}_{S \mathcal{D}^{\backslash i}}^{(q_2,l)}$, then, similar to inequality \eqref{eqn:ineq-det-X} we have
\begin{equation}
    \begin{aligned}
        \|\Delta_{X^{(q_2, l)}_{S\mathcal{D}}} \|_F 
        &= \|\sigma(\bm{U}_S \bm{X}_{S \mathcal{D}}^{(q_2,l-1)} \mathcal{W}^{(q_2)}_{\mathcal{D}}) -
                \sigma(\bm{U}_S \bm{X}_{S \mathcal{D}^{\backslash i}}^{(q_2,l)} \mathcal{W}^{(q_2)}_{\mathcal{D}^{\backslash i}}) \|_F \\
        & \leq \rho \|\bm{U}_S \|_2  \|\bm{X}_{S \mathcal{D}}^{(q_2,l-1)} W^{(q_2, l-1)}_{\mathcal{D}} - \bm{X}_{S \mathcal{D}^{\backslash i}}^{(q_2,l-1)} W^{(q_2, l-1)}_{\mathcal{D}^{\backslash i}} \|_F \\
        & \leq \rho \|\bm{U}_S \|_2  \|\bm{X}_{S \mathcal{D}}^{(q_2,l-1)} W^{(q_2, l-1)}_{\mathcal{D}} - 
        \bm{X}_{S \mathcal{D}}^{(q_2,l-1)} W^{(q_2, l-1)}_{\mathcal{D}^{\backslash i}} + \bm{X}_{S \mathcal{D}}^{(q_2,l-1)} W^{(q_2, l-1)}_{\mathcal{D}^{\backslash i}} -
        \bm{X}_{S \mathcal{D}^{\backslash i}}^{(q_2,l-1)} W^{(q_2, l-1)}_{\mathcal{D}^{\backslash i}} \|_F \\
        & \leq \rho \|\bm{U}_S \|_2  \|\bm{X}_{S \mathcal{D}}^{(q_2,l-1)} (W^{(q_2, l-1)}_{\mathcal{D}} -  W^{(q_2, l-1)}_{\mathcal{D}^{\backslash i}}) +(\bm{X}_{S \mathcal{D}}^{(q_2,l-1)} - \bm{X}_{S \mathcal{D}^{\backslash i}}^{(q_2,l-1)}) W^{(q_2, l-1)}_{\mathcal{D}^{\backslash i}} \|_F \\
        & \leq \rho \|\bm{U}_S \|_2 ( \|\bm{X}_{S \mathcal{D}}^{(q_2,l-1)} \|_F \|W^{(q_2, l-1)}_{\mathcal{D}} -  W^{(q_2, l-1)}_{\mathcal{D}^{\backslash i}}\|_2 +\|\bm{X}_{S \mathcal{D}}^{(q_2,l-1)} - \bm{X}_{S \mathcal{D}^{\backslash i}}^{(q_2,l-1)} \|_F \|W^{(q_2, l-1)}_{\mathcal{D}^{\backslash i}} \|_2) \\
        &= \rho \|\bm{U}_S \|_2 \hat{\beta}_W \|\Delta_{X^{(q_2, l-1)}_{S\mathcal{D}}} \|_F + 
        \rho \|\bm{U}_S \|_2 \hat{\beta}_{\Delta W} \|\bm{X}_{S \mathcal{D}}^{(q_2,l-1)} \|_F \\
        & \leq  \rho^l \|\bm{U}_S\|_2^l \hat{\beta}_W^l \|\Delta_{X^{(q_2, 0)}_{S\mathcal{D}}} \|_F + \sum_{k=1}^l \rho^k \|\bm{U}_S\|_2^k \hat{\beta}_W^{k-1} \hat{\beta}_{\Delta W} \|\bm{X}_{S \mathcal{D}}^{(q_2,l-k)} \|_F
    \end{aligned}
\end{equation}
where $\|\Delta_{X^{(q_2, 0)}_{S\mathcal{D}}} \|_F = \|\bm{X}_S - \bm{X}_S \|_F = 0$. 
We can directly use the inequality \eqref{eqn:ineq-X}, such that 
\begin{equation}
    \begin{aligned}
        \|\bm{X}_{S\mathcal{D}}^{(m, l)} \|_F  \leq \rho^l \hat{\beta}_W^l \beta_X ( 1 + \beta_A)^l (1 + \alpha)^{-l}
    \end{aligned}
\end{equation}
Thus, we continue the proof
\begin{equation}
    \begin{aligned}
        \|\Delta_{X^{(q_2, l)}_{S\mathcal{D}}} \|_F 
        & \leq  \rho^l \|\bm{U}_S\|_2^l \hat{\beta}_W^l \|\Delta_{X^{(q_2, 0)}_{S\mathcal{D}}} \|_F + \sum_{k=1}^l \rho^k \|\bm{U}_S\|_2^k \hat{\beta}_W^{k-1} \hat{\beta}_{\Delta W} \|\bm{X}_{S \mathcal{D}}^{(q_2,l-k)} \|_F \\
        &\leq l \rho^l  (1 + \alpha)^{-l} ( 1 + \beta_A)^l \hat{\beta}_W^{l-1} \hat{\beta}_{\Delta W}   \beta_X 
    \end{aligned}
\end{equation}

Also similar to inequality \eqref{lem: ineq-subgraph}, we have
\begin{equation}
    \begin{aligned}
    \| \bm{z}_{\mathcal{D}}^{(q_2)} - \bm{z}_{\mathcal{D}^{\backslash i}}^{(q_2)}\| 
    &=\| F(\bm{A}_S, \bm{X}_S; \mathcal{W}^{(q_2)}_{\mathcal{D}}) - F(\bm{A}_S, \bm{X}_S; \mathcal{W}^{(q_2)}_{\mathcal{D}^{\backslash i}}) \| \\
    & = \frac{1}{n} \|\bm{1}^\top ( \bm{X}_{S \mathcal{D}}^{(q_2,L)}) - \bm{1}^\top ( \bm{X}_{S \mathcal{D}^{\backslash i}}^{(q_2,L)})\|\\
    & = \frac{1}{\sqrt{n}} \|\bm{X}_{S \mathcal{D}}^{(q_2,L)} -  \bm{X}_{S \mathcal{D}^{\backslash i}}^{(q_2,L)}\|_F = \frac{1}{\sqrt{n}} \|\Delta_{X^{(q_2, L)}_{S\mathcal{D}}}\|_F\\
    & \leq \frac{L}{\sqrt{n}} \rho^L  (1 + \alpha)^{-L} ( 1 + \beta_A)^L \hat{\beta}_W^{L-1} \hat{\beta}_{\Delta W}   \beta_X
    \end{aligned}
\end{equation}

Finally, we have
    \begin{equation}
        \begin{aligned}
        \|\bm{g} - \bm{g}^{\backslash i} \| 
        & = \|\bm{\lambda}_{\mathcal{D}}^\top \bm{Z}_{\mathcal{D}} - 
        \bm{\lambda}_{\mathcal{D}^{\backslash i}}^\top \bm{Z}_{\mathcal{D}^{\backslash i}}\|   \\
        & = \|\bm{\lambda}_{\mathcal{D}}^\top \bm{Z}_{\mathcal{D}} -\bm{\lambda}_{\mathcal{D}^{\backslash i}}^\top \bm{Z}_{\mathcal{D}}  + 
        \bm{\lambda}_{\mathcal{D}^{\backslash i}}^\top \bm{Z}_{\mathcal{D}} - 
        \bm{\lambda}_{\mathcal{D}^{\backslash i}}^\top \bm{Z}_{\mathcal{D}^{\backslash i}}\| \\
        & = \|(\bm{\lambda}_{\mathcal{D}} - \bm{\lambda}_{\mathcal{D}^{\backslash i}})^\top \bm{Z}_{\mathcal{D}} +  \bm{\lambda}_{\mathcal{D}^{\backslash i}}^\top 
        (\bm{Z}_{\mathcal{D}} - \bm{Z}_{\mathcal{D}^{\backslash i}})\| \\
        & \leq \|\bm{\lambda}_{\mathcal{D}} - \bm{\lambda}_{\mathcal{D}^{\backslash i}} \| \|\bm{Z}_{\mathcal{D}} \|_2 + \|\bm{\lambda}_{\mathcal{D}^{\backslash i}} \| 
        \|\bm{Z}_{\mathcal{D}} - \bm{Z}_{\mathcal{D}^{\backslash i}} \|_2 \\
        & \leq \|\bm{\lambda}_{\mathcal{D}} - \bm{\lambda}_{\mathcal{D}^{\backslash i}} \| 
        \| \bm{z}_{\mathcal{D}}^{(q_1)}\| + 
        \|\bm{\lambda}_{\mathcal{D}^{\backslash i}} \|
        \| \bm{z}_{\mathcal{D}}^{(q_2)} - \bm{z}_{\mathcal{D}^{\backslash i}}^{(q_2)}\| \\
        & \leq \|\bm{\lambda}_{\mathcal{D}} - \bm{\lambda}_{\mathcal{D}^{\backslash i}} \| 
        \frac{1}{\sqrt{n}} \rho^L \hat{\beta}_W^L \beta_X ( 1 + \beta_A)^L (1 + \alpha)^{-L} + \|\bm{\lambda}_{\mathcal{D}^{\backslash i}} \| \frac{L}{\sqrt{n}} \rho^L  (1 + \alpha)^{-L} ( 1 + \beta_A)^L \hat{\beta}_W^{L-1} \hat{\beta}_{\Delta W} \beta_X\\
        & = \frac{1}{\sqrt{n}} \rho^L \hat{\beta}_W^{L-1}  \beta_X ( 1 + \beta_A)^L (1 + \alpha)^{-L} \left( \hat{\beta}_W  \|\bm{\lambda}_{\mathcal{D}} - \bm{\lambda}_{\mathcal{D}^{\backslash i}} \| + L \hat{\beta}_{\Delta W} \|\bm{\lambda}_{\mathcal{D}^{\backslash i}} \|  \right)
        \end{aligned}
    \end{equation}
\end{proof}

\section{Experiment}\label{app:experiment}
In this section, we present additional details of the experiment.

\subsection{Supervised Learning}
We conduct supervised XGL via pattern analysis, the proposed PXGL-GNN, by solving optimization with the classification loss. The dataset is split into 80\% training, 10\% validation, and 10\% testing data. The weight parameter $\bm{\lambda}$, indicating each pattern's contribution to graph representation learning, is reported in Table \ref{tab:lambda-supervised}. We also visualize the graph representation $\bm{g}$ and three pattern representations $\bm{z}^{(m)}$ of PROTEINS. Results show the paths pattern is most important for learning $\bm{g}$, and the ensemble representation $\bm{g}$ outperforms single pattern representations $\bm{z}^{(m)}$.

\begin{table*}[h]
\centering
\caption{The learned $\bm{\lambda}$ of PXGL-GNN (supervised). The largest value is {\bf bold} and the second largest value is {\color{blue} blue}.}
\label{tab:lambda-supervised-appendix}
\resizebox{1\linewidth}{!}{
\setlength{\tabcolsep}{1.2mm}
\renewcommand{\arraystretch}{0.85}
\begin{tabular}{c|cccccccc}
\hline
Pattern     & MUTAG & PROTEINS & DD & NCI1 & COLLAB & IMDB-B & REDDIT-B & REDDIT-M5K \\ \hline
paths       & $\color{blue}0.095 \pm 0.014$ & $\bf0.550 \pm 0.070$ & $0.093 \pm 0.012$ & $0.022 \pm 0.002$ & $\bf0.587 \pm 0.065$ & $\color{blue}0.145 \pm 0.018$ & $0.131 \pm 0.027$ & $0.027 \pm 0.003$ \\
trees       & $0.046 \pm 0.005$ & $0.074 \pm 0.009$ & $0.054 \pm 0.006$ & $0.063 \pm 0.008$ & $0.105 \pm 0.013$ & $0.022 \pm 0.003$ & $0.055 \pm 0.007$ & $0.025 \pm 0.003$ \\
graphlets   & $0.062 \pm 0.008$ & $0.081 \pm 0.011$ & $\color{blue}0.125 \pm 0.015$ & $0.101 \pm 0.013$ & $0.063 \pm 0.008$ & $0.084 \pm 0.011$ & $0.026 \pm 0.003$ & $0.054 \pm 0.007$ \\
cycles      & $\bf 0.654 \pm 0.085$ & $0.099 \pm 0.013$ & $0.094 \pm 0.012$ & $\color{blue}0.176 \pm 0.022$ & $0.022 \pm 0.003$ & $0.123 \pm 0.016$ & $0.039 \pm 0.005$ & $0.037 \pm 0.005$ \\
cliques     & $0.082 \pm 0.011$ & $\color{blue}0.098 \pm 0.012$ & $\bf0.572 \pm 0.073$ & $\bf0.574 \pm 0.075$ & $\color{blue}0.134 \pm 0.017$ & $\bf0.453 \pm 0.054$ & $\color{blue} 0.279 \pm 0.069$ & $\color{blue}0.256 \pm 0.067$ \\
wheels      & $0.026 \pm 0.003$ & $0.039 \pm 0.005$ & $0.051 \pm 0.007$ & $0.012 \pm 0.002$ & $0.068 \pm 0.009$ & $0.037 \pm 0.004$ & $0.036 \pm 0.005$ & $0.023 \pm 0.003$ \\
stars       & $0.035 \pm 0.005$ & $0.056 \pm 0.007$ & $0.011 \pm 0.002$ & $0.052 \pm 0.007$ & $0.021 \pm 0.003$ & $0.136 \pm 0.017$ & $\bf 0.447 \pm 0.006$ & $\bf 0.578 \pm 0.033$ \\ 
\hline
\end{tabular}
}
\end{table*}

The compared baselines include classical GNNs like GIN \cite{xu2018powerful}, DiffPool \cite{ying2018hierarchical}, DGCNN \cite{zhang2018end}, GRAPHSAGE \cite{hamilton2017inductive}; subgraph-based GNNs like SubGNN \cite{kriege2012subgraph}, SAN \cite{zhao2018substructure}, SAGNN \cite{zeng2023substructure}; and recent methods like S2GAE \cite{tan2023s2gae} and ICL \cite{zhao2024twist}. The accuracies in Table \ref{tab:supervised-acc} show that our method performs the best.

\begin{table*}[h]
\centering
\caption{Accuracy (\%) of Graph Classification. The best accuracy is {\bf bold} and the second best is {\color{blue} blue}.}
\label{tab:supervised-acc-appendix}
\resizebox{1\linewidth}{!}{
\setlength{\tabcolsep}{1.2mm}
\renewcommand{\arraystretch}{0.85}
\begin{tabular}{c|cccccccc}
\hline
Method    & MUTAG                 & PROTEINS         & DD               & NCI1             & COLLAB           & IMDB-B           & REDDIT-B         & REDDIT-M5K  \\ \hline
GIN        & 84.53 $\pm$ 2.38     & 73.38 $\pm$2.16  & 76.38 $\pm$1.58  & 73.36 $\pm$1.78  & 75.83 $\pm$ 1.29 & 72.52 $\pm$ 1.62 & 83.27 $\pm$ 1.30 & 52.48 $\pm$  1.57\\
DiffPool   & 86.72 $\pm$ 1.95    & 76.07 $\pm$1.62  & 77.42 $\pm$2.14  & 75.42 $\pm$2.16  & 78.77 $\pm$ 1.36 & 73.55 $\pm$ 2.14  & 84.16 $\pm$ 1.28 & 51.39 $\pm$  1.48\\
DGCNN      & 84.29 $\pm$ 1.16   & 75.53 $\pm$2.14  & 76.57 $\pm$1.09  & 74.81 $\pm$1.53  & 77.59 $\pm$ 2.24 & 72.19 $\pm$ 1.97 & 86.33 $\pm$ 2.29 & 53.18 $\pm$  2.41\\
GRAPHSAGE  & 86.35 $\pm$ 1.31   & 74.21 $\pm$1.85  & 79.24 $\pm$2.25  & 77.93 $\pm$2.04  & 76.37 $\pm$ 2.11 & 73.86 $\pm$ 2.17  & 85.59 $\pm$ 1.92 & 51.65 $\pm$  2.55\\
SubGNN     & 87.52 $\pm$ 2.37   & 76.38 $\pm$1.57  & 82.51 $\pm$1.67  & 82.58 $\pm$1.79  & 81.26 $\pm$ 1.53 & 71.58 $\pm$ 1.20  & 88.47 $\pm$ 1.83 & 53.27 $\pm$  1.93\\
SAN        & 92.65 $\pm$ 1.53   & 75.62 $\pm$2.39  & 81.36 $\pm$2.10  &\color{blue} 83.07 $\pm$1.54  &\textbf{\color{blue}} 82.73 $\pm$ 1.92 & 75.27 $\pm$ 1.43  & 90.38 $\pm$ 1.54 & 55.49 $\pm$  1.75\\
SAGNN      &\color{blue} 93.24 $\pm$ 2.51   & 75.61 $\pm$2.28  & 84.12 $\pm$1.73  & 81.29 $\pm$1.22  & 79.94 $\pm$ 1.83 & 74.53 $\pm$ 2.57 & 89.57 $\pm$ 2.13 & 54.11 $\pm$  1.22\\
ICL        & 91.34 $\pm$ 2.19   & 75.44 $\pm$1.26  & 82.77 $\pm$1.42  & 83.45 $\pm$1.78  & 81.45 $\pm$ 1.21 & 73.29 $\pm$ 1.46   &\color{blue} 90.13 $\pm$ 1.40 &\color{blue} 56.21 $\pm$  1.35\\
S2GAE      & 89.27 $\pm$ 1.53   &\color{blue} 76.47 $\pm$1.12  &\color{blue} 84.30 $\pm$1.77  & 82.37 $\pm$2.24  & 82.35 $\pm$ 2.34 &\color{blue} 75.77 $\pm$ 1.72   & 90.21 $\pm$ 1.52 & 54.53 $\pm$  2.17\\ \hline
PXGL-GNN       &\bf 94.87 $\pm$ 2.26   &\bf 78.23 $\pm$2.46  &\bf 86.54 $\pm$1.95  &\bf 85.78 $\pm$2.07  &\bf 83.96 $\pm$ 1.59 &\bf 77.35 $\pm$ 2.32   &\bf 91.84 $\pm$ 1.69 &\bf 57.36 $\pm$  2.14 \\ \hline
\end{tabular} 
}
\end{table*}

\subsection{Unsupervised Learning}
We conduct unsupervised XGL via pattern analysis, the proposed PXGL-GNN, by solving optimization with the KL divergence loss.
The weight parameter $\bm{\lambda}$ for XGL is reported in Table \ref{tab:lambda-unsupervised}.
Results show that the ensemble representation $\bm{g}$ outperforms single pattern representations $\bm{z}^{(m)}$.

\begin{table*}[h]
\centering
\caption{The learned $\bm{\lambda}$ of PXGL-GNN (unsupervised). The largest value is {\bf bold} and the second largest value is {\color{blue} blue}.}
\label{tab:lambda-unsupervised}
\resizebox{1\linewidth}{!}{
\setlength{\tabcolsep}{1.2mm}
\renewcommand{\arraystretch}{0.85}
\begin{tabular}{c|ccccccccc}
\hline
Pattern    & MUTAG & PROTEINS & DD & NCI1 & COLLAB & IMDB-B & REDDIT-B & REDDIT-M5K \\ \hline
paths       & $0.085 \pm 0.021$ & $\bf0.463 \pm 0.057$ & $0.083 \pm 0.010$ & $0.023 \pm 0.001$ & $\bf0.478 \pm 0.046$ & $0.153 \pm 0.018$ & $0.101 \pm 0.007$ & $0.084 \pm 0.006$ \\
trees       & $0.027 \pm 0.005$ & $0.082 \pm 0.008$ & $0.069 \pm 0.007$ & $0.042 \pm 0.002$ & $0.127 \pm 0.017$ & $0.082 \pm 0.009$ & $0.060 \pm 0.003$ & $0.036 \pm 0.002$ \\
graphlets   & $0.074 \pm 0.009$ & $0.085 \pm 0.010$ & $\color{blue}0.172 \pm 0.020$ & $0.105 \pm 0.012$ & $0.055 \pm 0.006$ & $0.098 \pm 0.011$ & $0.025 \pm 0.002$ & $0.055 \pm 0.005$ \\
cycles      & $\bf0.546 \pm 0.065$ & $0.095 \pm 0.011$ & $0.108 \pm 0.013$ & $\color{blue}0.276 \pm 0.033$ & $0.022 \pm 0.002$ & $\color{blue}0.124 \pm 0.014$ & $0.043 \pm 0.005$ & $0.028 \pm 0.003$ \\
cliques     & $\color{blue}0.197 \pm 0.023$ & $\color{blue}0.207 \pm 0.025$ & $\bf0.527 \pm 0.063$ & $\bf0.482 \pm 0.058$ & $\color{blue}0.243 \pm 0.029$ & $\bf0.423 \pm 0.051$ & $ \color{blue}0.212 \pm 0.061$ & $\color{blue}0.157 \pm 0.067$ \\
wheels      & $0.032 \pm 0.003$ & $0.036 \pm 0.004$ & $0.018 \pm 0.002$ & $0.013 \pm 0.001$ & $0.044 \pm 0.005$ & $0.035 \pm 0.004$ & $0.036 \pm 0.003$ & $0.025 \pm 0.013$ \\
stars       & $0.039 \pm 0.004$ & $0.032 \pm 0.002$ & $0.023 \pm 0.003$ & $0.059 \pm 0.007$ & $0.031 \pm 0.001$ & $0.085 \pm 0.010$ & $\bf0.455 \pm 0.019$ & $ \bf 0.585 \pm 0.022$ \\
\hline
\end{tabular}
}
\end{table*}

For clustering performance, we use clustering accuracy (ACC) and Normalized Mutual Information (NMI). Baselines include four kernels: Random walk kernel (RW) \cite{borgwardt2005protein}, Sub-tree kernels \cite{da2012tree,smola2002fast}, Graphlet kernels \cite{prvzulj2007biological}, Weisfeiler-Lehman (WL) kernels \cite{kriege2012subgraph}; and three unsupervised graph representation learning methods with Gaussian kernel: InfoGraph \cite{sun2019infograph}, GCL \cite{you2020graph}, 
% AutoGCL \cite{yin2022autogcl}, 
GraphACL \cite{luo2023self}. The results are in Table \ref{tab:unsupervised-nmi}. Our method outperformed all competitors in almost all cases.

\begin{table*}[h!]
\centering
\caption{ACC and NMI of Graph Clustering. The best ACC is {\bf bold} and the the second best ACC is {\color{blue} blue}. The best NMI is {\color{green} green} and the second best NMI is with $^*$.}
\label{tab:unsupervised-nmi}
\resizebox{1\linewidth}{!}{
\setlength{\tabcolsep}{1.2mm}
\renewcommand{\arraystretch}{0.85}
\begin{tabular}{c|c|cccccccc}
\hline
Method  &Metric  & MUTAG & PROTEINS & DD & NCI1 & COLLAB & IMDB-B & REDDIT-B& REDDIT-M5K  \\ \hline
\multirow{2}{*}{\makecell{RW}}        
            & ACC & 0.724 $\pm$0.023 & 0.718 $\pm$ 0.019  & 0.529 $\pm$ 0.017 & 0.519 $\pm$0.025 &\color{blue} 0.596 $\pm$0.019 & 0.669 $\pm$0.028 & $\geq$ 1 day  & $\geq$ 1 day \\
            & NMI & 0.283 $\pm$0.008 & 0.226 $\pm$ 0.008 & 0.207 $\pm$ 0.003  & 0.218 $\pm$0.009 & 0.356$^*$ $\pm$0.002  & 0.295  $\pm$0.006  & $\geq$ 1 day  & $\geq$ 1 day \\ \hline
\multirow{2}{*}{\makecell{sub-tree}}      
            & ACC & 0.716 $\pm$0.017  & 0.683 $\pm$ 0.023 & 0.563 $\pm$ 0.026 & 0.532 $\pm$0.016 & 0.533 $\pm$0.021  & 0.627 $\pm$0.022  & $\geq$ 1 day  & $\geq$ 1 day \\
            & NMI & 0.217 $\pm$0.005 & 0.167 $\pm$  0.004 & 0.225 $\pm$ 0.005 & 0.295 $\pm$0.004  & 0.198 $\pm$0.005  & 0.254  $\pm$0.007  & $\geq$ 1 day  & $\geq$ 1 day \\ \hline
\multirow{2}{*}{\makecell{Graphlet}}      
            & ACC & 0.727 $\pm$0.020 & 0.654 $\pm$  0.017 &\bf 0.581 $\pm$ 0.014 & 0.526 $\pm$0.032 & 0.525 $\pm$0.026  & 0.617 $\pm$0.019  & $\geq$ 1 day  & $\geq$ 1 day \\
            & NMI & 0.225 $\pm$0.003 & 0.131 $\pm$  0.009 &\color{green} 0.320 $\pm$ 0.009 & 0.273 $\pm$0.005  & 0.217 $\pm$0.003  & 0.210 $\pm$0.004  & $\geq$ 1 day  & $\geq$ 1 day \\ \hline
\multirow{2}{*}{\makecell{WL}}          
            & ACC & 0.695 $\pm$0.031 & 0.647 $\pm$  0.032 & 0.517 $\pm$ 0.020 & 0.517 $\pm$0.028  & 0.569 $\pm$0.017  & 0.635 $\pm$0.017  & $\geq$ 1 day  & $\geq$ 1 day \\
            & NMI & 0.185 $\pm$0.007 & 0.135 $\pm$  0.001 & 0.192 $\pm$ 0.008 & 0.234 $\pm$0.007  & 0.253  $\pm$0.007  & 0.261 $\pm$0.003  & $\geq$ 1 day  & $\geq$ 1 day \\ \hline
\multirow{2}{*}{\makecell{InfoGraph}}     
            & ACC & 0.729 $\pm$0.021 & 0.716 $\pm$ 0.019 & 0.549 $\pm$ 0.035 & 0.535 $\pm$0.012 & 0.597  $\pm$0.020  & 0.624 $\pm$0.016  & 0.582 $\pm$0.023  &\color{blue} 0.597 $\pm$0.019   \\
            & NMI & 0.236 $\pm$0.005 & 0.231 $\pm$ 0.003 & 0.266 $\pm$ 0.004 & 0.263 $\pm$0.005 & 0.311  $\pm$0.008  & 0.198 $\pm$0.005  & 0.206 $\pm$0.006  & 0.286$^*$ $\pm$0.006   \\ \hline
\multirow{2}{*}{\makecell{GCL }}       
            & ACC &\color{blue} 0.761 $\pm$0.014 & 0.723 $\pm$ 0.025 & 0.563 $\pm$ 0.016 &\color{blue} 0.558 $\pm$0.010 & 0.582  $\pm$0.015  & 0.653 $\pm$0.024  & 0.573 $\pm$0.015  & 0.582 $\pm$0.017  \\
            & NMI & 0.337 $\pm$0.003 & 0.258 $\pm$ 0.002 & 0.289 $\pm$ 0.009 &\color{green} 0.341 $\pm$0.002 & 0.293 $\pm$ 0.009 & 0.253 $\pm$0.008  & 0.195 $\pm$0.005  & 0.266 $\pm$0.005  \\ \hline
\multirow{2}{*}{\makecell{GraphACL }}      
            & ACC & 0.742 $\pm$0.023 &\color{blue} 0.731 $\pm$ 0.027 & 0.572 $\pm$ 0.027 & 0.522 $\pm$0.013 & 0.554  $\pm$0.013  &\color{blue} 0.679 $\pm$0.013   &\color{blue} 0.594 $\pm$0.014   & 0.567 $\pm$0.023   \\   
            & NMI & 0.347$^*$ $\pm$0.007 & 0.274$^*$ $\pm$ 0.008 & 0.312 $\pm$ 0.003 & 0.260 $\pm$0.007 & 0.236 $\pm$0.006  & 0.315$^*$ $\pm$0.007  & 0.215$^*$ $\pm$ 0.006  & 0.238 $\pm$ 0.009 \\ \hline
\multirow{2}{*}{\makecell{PXGL-GNN}}         
            & ACC &\bf 0.778 $\pm$0.029 &\bf 0.746 $\pm$ 0.019 &\color{blue} 0.576 $\pm$ 0.035 &\bf 0.564 $\pm$0.013 &\bf 0.612  $\pm$0.014 &\bf 0.686 $\pm$0.027 &\bf 0.616 $\pm$0.017  &\bf 0.608 $\pm$0.023 \\ 
            & NMI &\color{green} 0.352 $\pm$0.006 &\color{green} 0.292 $\pm$ 0.010 & 0.317$^*$ $\pm$ 0.003 & 0.327$^*$ $\pm$0.008 &\color{green} 0.372 $\pm$0.007 &\color{green} 0.324 $\pm$0.011 & 0.224 $\pm$ 0.009  &\color{green} 0.295 $\pm$0.012 \\ \hline
\end{tabular} 
}
\end{table*}

\subsection{Additional Evaluation on the Ensemble Kernel (PXGL-EGK)}
Here, we compare our ensemble kernel (PXGL-EGK) with individual kernels $K_{\mathcal{P}}$. We report the results as follows. Specifically, we use three pattern counting kernels in the ensemble method: Random Walk (RW) kernels \cite{borgwardt2005protein,gartner2003graph}, Sub-tree kernels \cite{da2012tree,smola2002fast}, and Graphlet kernels \cite{prvzulj2007biological}. Since graph kernels are unsupervised learning methods, we compare the clustering accuracy and Normalized Mutual Information (NMI) of each kernel, as shown in Table \ref{tab:unsupervised-nmi-ker}. The result shows that PXGL-EGK outperforms each individual kernel it used.

\begin{table*}[h!]
\centering
\caption{ACC and NMI of Graph Clustering. The best ACC is \textbf{bold} and the best NMI is {\color{green} green}.}
\label{tab:unsupervised-nmi-ker}
\resizebox{0.65\linewidth}{!}{
\setlength{\tabcolsep}{1.2mm}
\renewcommand{\arraystretch}{0.85}
\begin{tabular}{c|c|cccccccc}
\hline
Method & Metric & MUTAG & PROTEINS & DD & IMDB-B \\ \hline
\multirow{2}{*}{RW}         & ACC & 0.743 $\pm$ 0.052 & 0.712 $\pm$ 0.021 & 0.516 $\pm$ 0.015 & 0.658 $\pm$ 0.014 \\
                            & NMI & 0.238 $\pm$ 0.016 & 0.268 $\pm$ 0.016 & 0.187 $\pm$ 0.002 & 0.266 $\pm$ 0.019 \\ \hline
\multirow{2}{*}{Sub-tree}   & ACC & 0.729 $\pm$ 0.013 & 0.692 $\pm$ 0.027 & 0.542 $\pm$ 0.016 & 0.612 $\pm$ 0.018 \\
                            & NMI & 0.195 $\pm$ 0.047 & 0.151 $\pm$ 0.028 & 0.229 $\pm$ 0.015 & 0.242 $\pm$ 0.013 \\ \hline
\multirow{2}{*}{Graphlet}   & ACC & 0.735 $\pm$ 0.026 & 0.636 $\pm$ 0.017 & 0.568 $\pm$ 0.013 & 0.614 $\pm$ 0.012 \\
                            & NMI & 0.214 $\pm$ 0.019 & 0.154 $\pm$ 0.026 & 0.285 $\pm$ 0.011 & 0.214 $\pm$ 0.025 \\ \hline
\multirow{2}{*}{PXGL-EGK}   & ACC &\bf 0.761 $\pm$ 0.025 &\bf 0.721 $\pm$ 0.028 &\bf 0.572 $\pm$ 0.025 &\bf 0.672 $\pm$ 0.023 \\
                            & NMI &\color{green} 0.328 $\pm$ 0.046 &\color{green} 0.321 $\pm$ 0.019 & \color{green}0.296 $\pm$ 0.013 &\color{green} 0.310 $\pm$ 0.021 \\ \hline
\end{tabular}
}
\end{table*}

\end{document}